\documentclass[final,12pt]{arxiv}

\usepackage{times}
\usepackage{mathtools}
\usepackage[mathscr]{euscript}
\usepackage{thmtools}
\usepackage{thm-restate}
\usepackage{booktabs}  
\usepackage{caption}
\usepackage{graphicx}
\usepackage{wrapfig}

\usepackage{bm}

\let\vec\undefined
\usepackage{MnSymbol}
\DeclareMathAlphabet\mathbb{U}{msb}{m}{n}
\usepackage{xpatch}

\def\Rset{\mathbb{R}}

\let\P\undefined

\DeclareMathOperator*{\P}{\mathbb{P}}
\DeclareMathOperator*{\E}{\mathbb E}
\DeclareMathOperator*{\argmax}{argmax}
\DeclareMathOperator*{\argmin}{argmin}

\DeclarePairedDelimiter{\abs}{\lvert}{\rvert} 
\DeclarePairedDelimiter{\bracket}{[}{]}
\DeclarePairedDelimiter{\curl}{\{}{\}}

\DeclarePairedDelimiter{\paren}{(}{)}

\newcommand{\sA}{{\mathscr A}}

\newcommand{\sC}{{\mathscr C}}
\newcommand{\sD}{{\mathscr D}}
\newcommand{\sE}{{\mathscr E}}
\newcommand{\sF}{{\mathscr F}}

\newcommand{\sH}{{\mathscr H}}
\newcommand{\sI}{{\mathscr I}}

\newcommand{\sM}{{\mathscr M}}
\newcommand{\sO}{{\mathscr O}}

\newcommand{\sR}{{\mathscr R}}

\newcommand{\sT}{{\mathscr T}}
\newcommand{\sU}{{\mathscr U}}

\newcommand{\sX}{{\mathscr X}}
\newcommand{\sY}{{\mathscr Y}}

\newcommand{\sfL}{{\mathsf L}}
\newcommand{\sfH}{{\mathsf H}}

\newcommand{\Rad}{\mathfrak R}

\newcommand{\h}{\widehat}
\newcommand{\ov}{\overline}
\newcommand{\wt}{\widetilde}
\newcommand{\e}{\epsilon}

\newcommand{\ignore}[1]{}

\DeclareMathOperator{\sign}{sign}
\def\Nset{\mathbb{N}}
\newcommand{\hh}{{\sf h}}

\usepackage[toc, page, header]{appendix}
\title[Universal Growth Rate]{A Universal Growth Rate for Learning with Smooth Surrogate Losses}

\arxivauthor{%
 \Name{Anqi Mao} \Email{aqmao@cims.nyu.edu}\\
 \addr Courant Institute of Mathematical Sciences, New York%
 \AND
 \Name{Mehryar Mohri} \Email{mohri@google.com}\\
 \addr Google Research and Courant Institute of Mathematical Sciences, New York%
 \AND
 \Name{Yutao Zhong} \Email{yutao@cims.nyu.edu}\\
 \addr Courant Institute of Mathematical Sciences, New York%
}

\begin{document}

\maketitle

\begin{abstract}

This paper presents a comprehensive analysis of the growth rate of
$\sH$-consistency bounds (and excess error bounds) for various
surrogate losses used in classification. We prove a square-root growth
rate near zero for smooth margin-based surrogate losses in binary
classification, providing both upper and lower bounds under mild
assumptions. This result also translates to excess error bounds.  Our
lower bound requires weaker conditions than those in previous work for
excess error bounds, and our upper bound is entirely novel.  Moreover,
we extend this analysis to multi-class classification with a series of
novel results, demonstrating a universal square-root growth rate for
smooth \emph{comp-sum} and \emph{constrained losses}, covering common
choices for training neural networks in multi-class classification.
Given this universal rate, we turn to the question of choosing among
different surrogate losses.  We first examine how $\sH$-consistency
bounds vary across surrogates based on the number of classes. Next,
ignoring constants and focusing on behavior near zero, we identify
\emph{minimizability gaps} as the key differentiating factor in these
bounds. Thus, we thoroughly analyze these gaps, to guide surrogate
loss selection, covering: comparisons across different comp-sum
losses, conditions where gaps become zero, and general conditions
leading to small gaps.  Additionally, we demonstrate the key role of
minimizability gaps in comparing excess error bounds and
$\sH$-consistency bounds.
\end{abstract}


\section{Introduction}

Learning algorithms frequently optimize surrogate loss functions like
the logistic loss, in lieu of the task's true objective, commonly the
zero-one loss.  This is necessary when the original loss function is
computationally intractable to optimize or lacks essential
mathematical properties such as differentiability. But, what
guarantees can we rely on when minimizing a surrogate loss? This is a
fundamental question with significant implications for learning.

The related property of Bayes-consistency of surrogate losses has been
extensively studied in the context of binary classification.
\citet{Zhang2003}, \citet{bartlett2006convexity} and
\citet{steinwart2007compare} established Bayes-consistency for various
convex loss functions, including margin-based surrogates. They also
introduced excess error bounds (or surrogate regret bounds) for
margin-based surrogates.  \citet{reid2009surrogate} extended these
results to proper losses in binary classification.

The Bayes-consistency of several surrogate loss function families in
the context of multi-class classification has also been studied
by \citet{zhang2004statistical} and
\citet{tewari2007consistency}. \cite{zhang2004statistical} established
a series of results for various multi-class classification
formulations, including negative results for multi-class hinge loss
functions \citep{crammer2001algorithmic}, as well as positive results
for the sum exponential loss
\citep{WestonWatkins1999,AwasthiMaoMohriZhong2022multi}, the
(multinomial) logistic loss
\citep{Verhulst1838,Verhulst1845,Berkson1944,Berkson1951}, and the
constrained losses \citep{lee2004multicategory}. Later,
\citet{tewari2007consistency} adopted a different geometric method to
analyze Bayes-consistency, yielding similar results for these loss
function families.  \citet{steinwart2007compare} developed general
tools to characterize Bayes consistency for both binary and
multi-class classification. Additionally, excess error bounds have
been derived by \cite{AvilaPiresSzepesvariGhavamzadeh2013} for a
family of constrained losses and by \citet{DuchieKhosraviRuan2018} for
loss functions related to generalized entropies.

For a surrogate loss $\ell$, an excess error bound holds for any
predictor $h$ and has the form $\sE_{\ell_{0-1}}\! (h) -
\sE_{\ell_{0-1}}^*\!\! \leq \Psi(\sE_{\ell}(h) - \sE_{\ell}^*)$, where
$\sE_{\ell_{0-1}}\!(h)$ and $\sE_{\ell}(h)$ represent the expected
losses of $h$ for the zero-one loss and surrogate loss respectively,
$\sE_{\ell_{0-1}}^*\!$ and $\sE_{\ell}^*$ the Bayes errors for the
zero-one and surrogate loss respectively, and $\Psi$ a non-decreasing
function. The \emph{growth rate} of excess error bounds, that is the
behavior of function $\Psi$ near zero, has gained attention in recent
research \citep{mahdavi2014binary,zhang2021rates,
frongillo2021surrogate,bao2023proper}.  \cite{mahdavi2014binary}
examined the growth rate for \emph{smoothed hinge losses} in binary
classification, demonstrating that smoother losses result in worse
growth rates. The optimal rate is achieved with the standard hinge
loss, which exhibits linear growth.  \cite{zhang2021rates} tied the
growth rate of excess error bounds in binary classification to two
properties of the surrogate loss function: consistency intensity and
conductivity.  These metrics enable comparisons of growth rates across
different surrogates. This prompts a natural question: can we
establish rigorous lower and upper bounds for excess error growth
rates under specific regularity conditions?

\citet{frongillo2021surrogate} pioneered research on this question in
binary classification settings.  They established a critical
square-root lower bound for excess error bounds when a surrogate loss
is locally strongly convex and has a locally Lipschitz gradient.
Additionally, they demonstrated a linear excess error bound for
Bayes-consistent polyhedral loss functions (convex and
piecewise-linear) \citep{finocchiaro2019embedding} (see also
\citep{lapin2016loss,ramaswamy2018consistent,yu2018lovasz,
  yang2020consistency}).  More recently, \citet{bao2023proper}
complemented these results by showing that proper losses associated
with Shannon entropy, exponential entropy, spherical entropy, squared
$\alpha$-norm entropies and $\alpha$-polynomial entropies, with
$\alpha > 1$, also exhibit a square-root lower bound for excess error
bounds relative to the $\ell_1$-distance.

However, while Bayes-consistency and excess error bounds are valuable,
they are not sufficiently informative, as they are established for the
family of all measurable functions and disregard the crucial role
played by restricted hypothesis sets in learning. As pointed out by
\citet{long2013consistency}, in some cases, minimizing
Bayes-consistent losses can result in constant expected error, while
minimizing inconsistent losses can yield an expected loss approaching
zero.  To address this limitation, the authors introduced the concept
of \emph{realizable $\sH$-consistency}, further explored by
\citet{KuznetsovMohriSyed2014} and \citet{zhang2020bayes}.
Nonetheless, these guarantees are only asymptotic and rely on a strong
realizability assumption that typically does not hold in practice.

Recent research by \citet{awasthi2022h,AwasthiMaoMohriZhong2022multi}
and \citet{mao2023cross,MaoMohriZhong2023ranking,
MaoMohriZhong2023structured,MaoMohriZhong2023characterization} has
instead introduced and analyzed \emph{$\sH$-consistency bounds}.
These bounds are more informative than Bayes-consistency since they
are hypothesis set-specific and non-asymptotic.  Their work covers
broad families of surrogate losses in binary classication, multi-class
classification, structured prediction, and abstention
\citep{MaoMohriMohriZhong2023twostage}.  Crucially, they provide upper
bounds on the \emph{estimation error} of the target loss, for example,
the zero-one loss in classification, that hold for any predictor $h$
within a hypothesis set $\sH$. These bounds relate this estimation
error to the surrogate loss estimation error.

Their general form is: $\sE_{\ell_{0-1}}(h) - \sE^*_{\ell_{0-1}}(\sH)
+ \sM_{\ell_{0-1}}(\sH) \leq \Gamma\paren*{\sE_{\ell}(h)
  -\sE^*_{\ell}(\sH) + \sM_{\ell}(\sH)}$, where
$\sE^*_{\ell_{0-1}}\!(\sH)$ and $\sE^*_{\ell}(\sH)$ represent the
best-in-class expected losses for the zero-one and surrogate loss
respectively, $\Gamma$ is a non-negative concave function and
$\sM_{\ell_{0-1}}(\sH)$ and $\sM_{\ell}(\sH)$ are \emph{minimizability
gaps}. The exact definition of these gaps will be detailed later.  For
now, let us mention that they are non-negative quantities,
upper-bounded by the approximation error of their respective loss
functions.  $\sH$-consistency bounds subsume excess error bounds as a
special case when the hypothesis set is expanded to include all
measurable functions, in which case the minimizability gaps
vanish. More generally, an $\sH$-consistency bound with a $\Gamma$
function implies $\sE_{\ell_{0-1}}(h) - \sE^*_{\ell_{0-1}}(\sH) +
\sM_{\ell_{0-1}}(\sH) \leq \Gamma\paren*{\sE_{\ell}(h)
  -\sE^*_{\ell}(\sH)} + \Gamma\paren*{\sM_{\ell}(\sH)}$ since a
concave function $\Gamma$ with $\Gamma(0) \geq 0$ is sub-additive over
$\Rset_+$. Thus, when the surrogate estimation loss $\sE_{\ell}(h) -
\sE^*_{\ell}(\sH)$ is minimized to $\e$, the zero-one estimation error
$\sE_{\ell_{0-1}}(h) - \sE^*_{\ell_{0-1}}(\sH)$ is bounded by
$\Gamma(\e) + \Gamma\paren*{\sM_{\ell}(\sH)} - \sM_{\ell_{0-1}}(\sH)$.
Can we characterize the growth rate of $\sH$-consistency bounds, that
is how quickly the functions $\Gamma$ increase near zero?

\textbf{Our results}. This paper presents a comprehensive analysis of
the growth rate of $\sH$-consistency bounds for all margin-based
surrogate losses in binary classification, as well as for
\emph{comp-sum losses} and \emph{constrained losses} in multi-class
classification.
We establish a square-root growth rate near zero for margin-based
surrogate losses $\ell$ defined by $\ell(h, x, y) = \Phi(-yh(x))$,
assuming only that $\Phi$ is convex and twice continuously
differentiable with $\Phi'(0) \!>\! 0$ and $\Phi''(0) \!>\! 0$
(Section~\ref{sec:binary}).  This includes both upper and lower bounds
(Theorem~\ref{thm:binary-lower}). These results directly apply to
excess error bounds as well. Importantly, our lower bound requires
weaker conditions than \citep[Theorem~4]{frongillo2021surrogate}, and
our upper bound is entirely novel.  This work demonstrates that the
$\sH$-consistency bound growth rate for these loss functions is
precisely square-root, refining the ``at least square-root'' finding of
these authors (for excess error bounds). It is known that polyhedral
losses admit a linear grow rate \citep{frongillo2021surrogate}. Thus,
a striking dichotomy emerges that reflects previous observations by
these authors: $\sH$-consistency bounds for polyhedral losses exhibit
a linear growth rate in binary classification, while they follow a
square-root rate for smooth loss functions.

Moreover, we significantly extend our findings to key multi-class
surrogate loss families, including \emph{comp-sum losses}
\citep{mao2023cross} (e.g., logistic loss or cross-entropy with
softmax \citep{Berkson1944}, sum-losses \citep{WestonWatkins1999},
generalized cross entropy loss \citep{zhang2018generalized}), and
\emph{constrained losses}
\citep{lee2004multicategory,AwasthiMaoMohriZhong2022multi}
(Section~\ref{sec:multi}).
In Section~\ref{sec:comp}, we prove that the growth rate of
$\sH$-consistency bounds for comp-sum losses is exactly
square-root. This applies when the auxiliary function $\Phi$ they are
based upon is convex and twice continuously differentiable with
$\Phi'(u) \!<\! 0$ and $\Phi''(u) \!>\! 0$ for all $u$ in $(0,
  \frac{1}{2}]$. These conditions hold for all common loss functions
used in practice.
Further, in Section~\ref{sec:cstnd}, we demonstrate that the
square-root growth rate also extends to $\sH$-consistency bounds for
constrained losses.  This requires the auxiliary function $\Phi$ to be
convex and twice continuously differentiable with $\Phi'(u) \!>\! 0$
and $\Phi''(u) \!>\! 0$ for any $u \geq 0$, alongside an additional
technical condition.  These are satisfied by all constrained losses
typically encountered in practice.

These results reveal a universal square-root growth rate for smooth
surrogate losses, the predominant choice in neural network training
(over polyhedral losses) for both binary and multi-class
classification in applications.
\ignore{
Crucially, under conditions on
\emph{minimizability gaps} detailed later, this implies a direct
relationship between the surrogate estimation loss and the target
zero-one estimation error: when the surrogate estimation loss is
reduced to a sufficiently small $\e > 0$, the zero-one estimation
error scales precisely as $\sqrt{\e}$. Our results apply to the growth
rate of the first surrogate estimation loss term
$\Gamma\paren*{\sE_{\ell}(h) -\sE^*_{\ell}(\sH)}$ (square-root), with
the second term $\Gamma\paren*{\sM_{\ell}(\sH)}$ being a constant.
}
Given this universal growth rate, how do we choose between different
surrogate losses? Section~\ref{sec:M-gaps} addresses this question
in detail. To start, we examine how $\sH$-consistency bounds vary
across surrogates based on the number of classes. Then, focusing on
behavior near zero (ignoring constants), we isolate minimizability
gaps as the key differentiating factor in these bounds.  These gaps
depend solely on the chosen surrogate loss and hypothesis set.
We provide a detailed analysis of minimizability gaps, covering:
comparisons across different comp-sum losses, conditions where gaps
become zero, and general conditions leading to small gaps.  These
findings help guide surrogate loss selection. Additionally, we
demonstrate the key role of minimizability gaps in comparing excess
error bounds and $\sH$-consistency bounds
(Appendix~\ref{app:excess-bounds}). Importantly, combining
$\sH$-consistency bounds with surrogate loss Rademacher complexity
bounds allows us to derive zero-one loss (estimation) learning bounds
for surrogate loss minimizers
(Appendix~\ref{app:generalization-bound}).

For a more comprehensive discussion of related work, please refer to
Appendix~\ref{app:related}. We start with the introduction of
necessary concepts and definitions.

\section{Preliminaries}
\label{sec:pre}

\textbf{Notation and definitions.} We denote the input space by $\sX$
and the label space by $\sY$, a finite set of cardinality $n$ with
elements $\curl*{1, \ldots, n}$.  $\sD$ denotes a distribution over
$\sX \times \sY$.

We write $\sH_{\rm{all}}$ to denote the family of all real-valued
measurable functions defined over $\sX \times \sY$ and denote by $\sH$
a subset, $\sH \subseteq \sH_{\rm{all}}$.  The label assigned by $h
\in \sH$ to an input $x \in \sX$ is denoted by $\hh(x)$ and defined by
$\hh(x) = \argmax_{y \in \sY} h(x, y)$, with an arbitrary but fixed
deterministic strategy used for breaking the ties. For simplicity, we
fix that strategy to be the one selecting the label with the highest
index under the natural ordering of labels.

We will consider general loss functions $\ell \colon \sH \times \sX
\times \sY \to \Rset_{+}$. For many loss functions used in practice, the
loss value at $(x, y)$, $\ell(h, x, y)$, only depends on the value $h$
takes at $x$ and not on its values on other
points. That is, there exists a measurable function $\hat \ell \colon \Rset^n \times \sY \to \Rset_{+}$ such that $\ell(h, x, y) = \hat \ell(h(x), y)$, where $h(x) = \bracket*{h(x, 1), \ldots, h(x, n)}$ is
the score vector of the predictor $h$. We will then say that $\ell$ is a \emph{pointwise loss
function}.
We denote by $\sE_\ell(h)$ the generalization error or expected loss
of a hypothesis $h \in \sH$ and by $\sE^*_\ell(\sH)$ the \emph{best-in
class error}:
$
\sE_\ell(h) = \E_{(x, y) \sim \sD}[\ell(h, x, y)],
\sE^*_\ell(\sH) = \inf_{h \in \sH} \sE_\ell(h).
$
$\sE^*_\ell\paren*{\sH_{\rm{all}}}$ is also
known as the \emph{Bayes error}.
We write $p(x, y) = \sD(Y = y \!\mid\! X = x)$ to denote the conditional
probability of $Y = y$ given $X = x$ and $p(x) = (p(x, 1),
\ldots, p(x, n))$ for the conditional probability vector for any $x \in
\sX$. We denote by $\sC_{\ell}(h, x)$ the \emph{conditional error} of
$h \in \sH$ at a point $x \in \sX$ and by $\sC_{\ell}^*(\sH, x)$ the
\emph{best-in-class conditional error}:
$
\sC_{\ell}(h, x) = \E_y \bracket*{\ell(h, x, y) \mid x} =
\sum_{y \in \sY} p(x, y) \, \ell(h, x, y),
\sC_{\ell}^*(\sH, x) = \inf_{h\in \sH}\sC_{\ell}(h, x),
$
and use the shorthand
$\Delta\sC_{\ell,\sH}(h,x) = \sC_{\ell}(h,x) - \sC_{\ell}^*(\sH, x)$
for the \emph{calibration gap} or \emph{conditional regret} for
$\ell$.
The generalization error of $h$ can be written as $\sE_{\ell}(h) =
\E_{x}\bracket*{\sC_{\ell}(h, x)}$. For convenience, we also
define, for any vector $p = (p_1, \ldots, p_n) \in \Delta^n$, where
$\Delta^n$ is the probability simplex of $\Rset^n$, $\sC_{\ell}(h, x,
p) = \sum_{y\in \sY} p_y \, \ell(h,x,y)$, $\sC_{\ell,\sH}^*(x, p) =
\inf_{h \in \sH} \sC_{\ell}(h, x, p)$ and $\Delta\sC_{\ell, \sH}(h, x,
p) = \sC_{\ell}(h, x, p) -\sC_{\ell,\sH}^*(x, p)$. Thus, we have
$\Delta\sC_{\ell,\sH}(h, x, p(x)) = \Delta\sC_{\ell, \sH}(h, x)$.

We will study the properties of a surrogate loss function $\ell_1$ for
a target loss function $\ell_2$.  In multi-class classification,
$\ell_2$ is typically the zero-one multi-class classification loss
function $\ell_{0-1}$ defined by $\ell_{0-1}(h, x, y) = 1_{\hh(x) \neq
  y}$. Some surrogate loss functions $\ell_1$ include the max losses \citep{crammer2001algorithmic}, comp-sum losses \citep{mao2023cross} and constrained losses \citep{lee2004multicategory}.

\textbf{Binary classification.} The definitions just presented were
given for the general multi-class classification setting.  In the
special case of binary classification (two classes), the standard
formulation and definitions are slightly different. For convenience,
the label space is typically defined as $\sY = \curl*{-1, +1}$.
Instead of two scoring functions, one for each label, a single
real-valued function is used whose sign determines the predicted
class. Thus, here, a hypothesis set $\sH$ is a family of measurable
real-valued functions defined over $\sX$ and $\sH_{\rm all}$ is the
family of all such functions. $\ell$ is \emph{pointwise} if there
exists a measurable function $\hat \ell \colon \Rset \times \sY \to
\Rset_{+}$ such that $\ell(h, x, y) = \hat \ell(h(x), y)$. The target
loss function is typically the binary loss $\ell_{0-1}$, defined by
$\ell_{0-1}(h, x, y) = 1_{\sign(h(x)) \neq y}$, where $\sign(h(x)) =
1_{h(x) \geq 0} - 1_{h(x) < 0}$. Some widely used surrogate losses
$\ell_1$ for $\ell_{0-1}$ are margin-based losses, which are defined
by $\ell_1(h, x, y) = \Phi\paren*{-yh(x)}$, for some non-decreasing
convex function $\Phi\colon \Rset \to \Rset_{+}$. Instead of two
conditional probabilities, one for each label, a single conditional
probability corresponding to the positive class $ +1 $ is used. That
is, let $\eta(x) = \sD(Y = + 1 \mid X = x)$ denote the conditional
probability of $Y = + 1$ given $X = x$. The conditional error can then
be expressed as:
\[
\sC_{\ell}(h,x) = \E_y \bracket*{\ell(h, x, y) \mid x} =
\eta(x) \ell(h, x, +1) + (1 - \eta(x)) \ell(h, x, -1).
\]
For convenience, we also define, for
any $p \in [0, 1]$, $\sC_{\ell}(h, x, p) = p \ell(h, x, +1) + (1 - p)\ell(h, x, -1)$, $\sC^*_{\ell, \sH}(x, p) = \inf_{h \in \sH} \sC_{\ell}(h, x, p)$  and $\Delta\sC_{\ell, \sH}(h, x, p) = \sC_{\ell}(h, x, p) - \inf_{h \in \sH}\sC_{\ell}(h, x, p)$. Thus, we have
$\Delta\sC_{\ell,\sH}(h, x, \eta(x)) = \Delta\sC_{\ell, \sH}(h, x)$.

To simplify matters, we will use the same notation for binary and
multi-class classification, such as $\sY$ for the label space or $\sH$
for a hypothesis set. We rely on the reader to adapt to the
appropriate definitions based on the context.

\textbf{Estimation, approximation, and excess errors.}  For a
hypothesis $h$, the difference
$\sE_{\ell}(h)-\sE_{\ell}^*\paren*{\sH_{\mathrm{all}}}$ is known as 
the \emph{excess error}. It can be decomposed into the sum of two
terms, the \emph{estimation error}, $\paren*{\sE_{\ell}(h) -
  \sE_{\ell}^*(\sH)}$ and the \emph{approximation error}
$\sA_{\ell}(\sH) = \paren*{\sE_{\ell}^*(\sH) -
  \sE_{\ell}^*\paren*{\sH_{\mathrm{all}}}}$:
\begin{equation}
\label{eq:excess-error-decomp} 
\begin{aligned}
    \sE_{\ell}(h) - \sE_{\ell}^*\paren*{\sH_{\mathrm{all}}}
    & = \paren*{\sE_{\ell}(h) - \sE_{\ell}^*(\sH)}
    + \paren*{\sE_{\ell}^*(\sH) - \sE_{\ell}^*\paren*{\sH_{\mathrm{all}}}}.
\end{aligned}
\end{equation}
\ignore{ The approximation error is thus the difference of best-in
  class error and Bayes error. It an be viewed as a measure of
  complexity: the larger the hypothesis set $\sH$, the smaller is
  $\sA_{\ell}(\sH)$.  }
A fundamental result for a pointwise loss function $\ell$ is that the
Bayes error and the approximation error admit the following simpler
expressions.  We give a concise proof of this lemma in
Appendix~\ref{app:lemma}, where we establish the measurability of the
function $x \mapsto \sC^*_{\ell}\paren*{\sH_{\rm{all}}, x}$.
\ignore{
For a pointwise loss function $\ell$, the Bayes error can be expressed in terms of the
expectation of the Bayes conditional error and thus the approximation
error can be rewritten as $\sA_{\ell}(\sH) = \sE^*_\ell(\sH) -
\E_{x}\bracket*{\sC_{\ell}^*(\sH_{\rm{all}},x)}$. The proof is included in Appendix~\ref{app:lemma}.
}
\begin{restatable}{lemma}{ApproximationError}
\label{lemma:approximation-error}
  Let $\ell$ be a pointwise loss function. Then, the Bayes error and
  the approximation error can be expressed as follows:
  $\sE^*_\ell\paren*{\sH_{\rm{all}}} =
  \E_{x}\bracket*{\sC_{\ell}^*(\sH_{\rm{all}},x)}$ and
  $\sA_{\ell}(\sH) = \sE^*_\ell(\sH) -
  \E_{x}\bracket*{\sC_{\ell}^*(\sH_{\rm{all}},x)}$.
\end{restatable}
For restricted hypothesis sets ($\sH\neq \sH_{\rm{all}}$), the
infimum's super-additivity implies that $\sE^*_\ell\paren*{\sH} \geq
\E_x \bracket*{\sC_\ell^*(\sH, x)}$.  This inequality is generally
strict, and the difference, $\sE^*_\ell\paren*{\sH} - \E_x
\bracket*{\sC_\ell^*(\sH, x)}$, plays a crucial role in our analysis.

\ignore{
Note that for a restricted hypothesis set $\sH\neq \sH_{\rm{all}}$, by
the super-additivity of the infimum, $\sE^*_\ell\paren*{\sH}$ is
always lower bounded by $\E_x \bracket*{\sC_\ell^*(\sH,
  x)}$. Moreover, the two terms are in general not equal. The
difference $\sE^*_\ell\paren*{\sH} - \E_x \bracket*{\sC_\ell^*(\sH,
  x)}$ will actually be a key quantity in our analysis.
}

\section{\texorpdfstring{$\sH$}{H}-consistency bounds}
\label{sec:min}

A widely used notion of consistency is that of
\emph{Bayes-consistency} given below
\citep{\ignore{Zhang2003,bartlett2006convexity,}steinwart2007compare}.

\begin{definition}[\textbf{Bayes-consistency}]
A loss function $\ell_1$ is \emph{Bayes-consistent} with respect to a
loss function $\ell_2$, if for any distribution $\sD$ and any sequence
$\{h_n\}_{n\in \Nset} \subset \sH_{\rm{all}}$, $\lim_{n \to +\infty}
\sE_{\ell_1}(h_n) - \sE_{\ell_1}^*\paren*{\sH_{\mathrm{all}}} = 0$
implies $\lim_{n \to +\infty} \sE_{\ell_2}(h_n) -
\sE_{\ell_2}^*\paren*{\sH_{\mathrm{all}}} = 0$.
\end{definition}
Thus, when this property holds, asymptotically, a nearly optimal
minimizer of $\ell_1$ over the family of all measurable functions is
also a nearly optimal optimizer of $\ell_2$.  But, Bayes-consistency
does not supply any information about a hypothesis set $\sH$ not
containing the full family $\sH_{\rm all}$, that is a typical
hypothesis set used for learning.  Furthermore, it is only an
asymptotic property and provides no convergence guarantee. In
particular, it does not give any guarantee for approximate
minimizers. Instead, we will consider upper bounds on the target
estimation error expressed in terms of the surrogate estimation error,
\emph{$\sH$-consistency bounds}
\citep{awasthi2022h,AwasthiMaoMohriZhong2022multi, mao2023cross},
which account for the hypothesis set $\sH$ adopted.

\begin{definition}[\textbf{$\sH$-consistency bounds}]
Given a hypothesis set $\sH$, an \emph{$\sH$-consistency bound}
relating the loss function $\ell_1$ to the loss function $\ell_2$ for
a hypothesis set $\sH$ is an inequality of the form
\begin{equation}
\label{eq:est-bound}
    \forall h \in \sH, \quad
\sE_{\ell_2}(h) - \sE^*_{\ell_2}(\sH)
+ \sM_{\ell_2}(\sH)
\leq \Gamma\paren*{\sE_{\ell_1}(h)
-\sE^*_{\ell_1}(\sH) + \sM_{\ell_1}(\sH)},
\end{equation}
that holds for any distribution $\sD$, where $\Gamma \colon \Rset_{+}
\to \Rset_{+}$ is a non-decreasing concave function with $\Gamma \geq
0$ \citep{awasthi2022h,AwasthiMaoMohriZhong2022multi}. Here,
$\sM_{\ell_1}(\sH)$ and $\sM_{\ell_2}(\sH)$ are \emph{minimizability
gaps} for the respective loss functions. The minimizability gap for a
hypothesis set $\sH$ and loss function $\ell$ is denoted by
$\sM_\ell(\sH)$ and defined as: $ \sM_\ell(\sH) =
\sE^*_\ell\paren*{\sH} - \E_x \bracket*{\sC_\ell^*(\sH, x)}$. It
quantifies the discrepancy between the best possible expected loss
within a hypothesis class and the expected infimum of pointwise
expected losses. This gap is always non-negative: $\sM_\ell(\sH) = \inf_{h
  \in \sH} \E_x[\sC_\ell(h, x)] - \E_x [\inf_{h \in \sH} \sC_\ell(\sH,
  x)] \geq 0$, by the infimum's super-additivity, and is bounded above
by the approximation error $\sA_{\ell}(\sH) = \inf_{h \in \sH}
\E_x[\sC_\ell(h, x)] - \E_x [\inf_{h \in \sH_{\rm{all}}} \sC_\ell(\sH,
  x)]$.  We further study the key role of minimizability gaps in
$\sH$-consistency bounds and their properties in
Section~\ref{sec:M-gaps} and Appendix~\ref{app:properties}.  As shown
in Appendix~\ref{app:explicit-form}, under general assumptions,
minimizability gaps are essential quantities required in any bound
that relates the estimation errors of two loss functions with an
arbitrary hypothesis set $\sH$.
\end{definition}
Thus, an $\sH$-consistency bound provides the guarantee that when the
surrogate estimation loss $\sE_{\ell}(h) - \sE^*_{\ell}(\sH)$ is
minimized to $\e$, the following upper bound holds for the zero-one
estimation error:
\begin{align*}
  \sE_{\ell}(h) - \sE^*_{\ell}(\sH)
  \leq \Gamma(\e + \sM_{\ell}(\sH)) - \sM_{\ell_{0-1}}(\sH)
  \leq \Gamma(\e) + \Gamma(\sM_{\ell}(\sH)) - \sM_{\ell_{0-1}}(\sH),
\end{align*}
where the second inequality follows from the sub-additivity of a
concave function $\Gamma$ over $\Rset_+$.  We will demonstrate that,
for smooth surrogate losses, $\Gamma(\e)$ scales as $\sqrt{\e}$.
Note, however, that, while $\Gamma(\e)$ tends to zero when $\e \to 0$
for functions $\Gamma$ derived in $\sH$-consistency bounds, the
remaining terms in the bound are constant. This is not surprising as,
in general, minimizing the surrogate estimation error to zero
\emph{cannot} guarantee that the zero-one estimation error will also
converge to zero.  This is well-known, for example, in the case of
linear models \citep{ben2012minimizing}. Instead, an $\sH$-consistency
bound provides the tightest possible upper bound on the estimation
error for the zero-one loss when the surrogate estimation error is
minimized.

The upper bound simplifies to $\Gamma(\e)$ when the minimizability
gaps are zero, which occurs when either $\sH = \sH_{\rm all}$ (the set
of all measurable functions) or in realizable cases, which are
particularly relevant to the practical use of complex neural networks
in applications.  In Appendix~\ref{app:small-M-gaps}, we examine more
general cases of small minimizability gaps, taking into account the
complexity of $\sH$ and the distribution.

Our results cover in particular the special case of excess bounds
($\sH = \sH_{\rm all}$). Let us emphasize that, for $\sH \neq \sH_{\rm
  all}$, $\sH$-consistency bounds offer tighter and more favorable
guarantees on the estimation error compared to those derived from
excess bounds analysis alone (see Appendix~\ref{app:excess-bounds}).

When $\ell_2 = \ell_{0-1}$, the zero-one loss, we say that
$\sT$ is the \emph{$\sH$-estimation error transformation function
of a surrogate loss $\ell$} if the following holds:
\[
\forall h \in \sH, \quad
\sT \paren*{\sE_{\ell_{0-1}}(h) - \sE^*_{\ell_{0-1}}(\sH)
+ \sM_{\ell_{0-1}}(\sH)}
\leq \sE_{\ell}(h)-\sE^*_{\ell}(\sH) + \sM_{\ell}(\sH),
\]
and the bound is \emph{tight}. That is, for any $t \in [0, 1]$, there
exists a hypothesis $h \in \sH$ and a distribution such that
$\sE_{\ell_{0-1}}(h) - \sE^*_{\ell_{0-1}}(\sH) + \sM_{\ell_{0-1}}(\sH)
= t$ and $\sE_{\ell}(h) - \sE^*_{\ell}(\sH) + \sM_{\ell}(\sH) =
\sT(t)$.
An explicit form of $\sT$ has been characterized for binary
margin-based losses \citep{awasthi2022h}, as well as comp-sum losses
and constrained losses in multi-class classification
\citep{MaoMohriZhong2023characterization}. In the following sections,
we will prove the property $\sT(t) = \Theta(t^2)$ (under mild
assumptions), demonstrating a square-root growth rate for
$\sH$-consistency bounds.\ignore{ Crucially, under conditions on
\emph{minimizability gaps} detailed later, this implies a direct
relationship between the surrogate estimation loss and the target
zero-one estimation error: when the surrogate loss is reduced to a
sufficiently small $\e > 0$, the zero-one error scales precisely as
$\sqrt{\e}$.} Appendix~\ref{app:bounds-example} provides
examples of $\sH$-consistency bounds for both binary and multi-class
classification.
Our analysis also suggests choosing appropriately $\sH$ and the
function $\Gamma$ to ensure a small minimizability gap and to take
into account the number of classes and other properties, as discussed
in Section~\ref{sec:M-gaps}.

\section{Binary classification}
\label{sec:binary}

We consider the broad family of margin-based loss functions $\ell$
defined for any $h \in \sH$, and $(x, y) \in \sX \times \sY$ by
$\ell(h, x, y) = \Phi(-y h(x))$, where $\Phi$ is a non-decreasing
convex function upper-bounding the zero-one loss. Margin-based loss
functions include most loss functions used in binary
classification. As an example, $\Phi(u) = \log(1 + e^{u})$ for the
logistic loss or $\Phi(u) = \exp(u)$ for the exponential loss. We say
that a hypothesis set $\sH$ is \emph{complete}, if for all $x \in
\sX$, we have $\curl*{h(x) \colon h \in \sH} = \Rset$. As shown by
\citet{awasthi2022h}, the transformation $\sT$ has the following form
for complete hypothesis sets:
\begin{equation*}
\sT(t) \colon = \inf_{u \leq 0} f_t(u) - \inf_{u \in \Rset} f_t(u).
\end{equation*}
Here, for any $t \in [0, 1]$, $f_t$ is defined by:
$\forall u \in \Rset,\, f_t(u) = \frac{1 - t}{2} \Phi(u) + \frac{1 + t}{2} \Phi(-u)$. The following result is useful for proving the growth rate in binary classification.

\begin{theorem}
\label{thm:binary-char}
Let $\sH$ be a complete hypothesis set.
Assume that $\Phi$ is convex and differentiable at zero and satisfies the inequality $\Phi'(0) > 0$. Then, the transformation $\sT$ can be expressed as follows:
\begin{equation*}
\forall t \in [0, 1], \quad \sT(t) = f_t(0) -\inf_{u \in \Rset} f_t(u).
\end{equation*}
\end{theorem}
\begin{proof}
By the convexity of $\Phi$, for any $t \in [0, 1]$ and $u \in \Rset_-$, we have
\[
f_t(u) = \frac{1 - t}{2} \Phi(u) + \frac{1 + t}{2} \Phi(-u)
\geq \Phi(0) - t u \Phi'(0) \geq \Phi(0).
\]
Thus,  we can write
$\sT(t)
= \inf_{u \leq 0} f_t(u) - \inf_{u \in \Rset} f_t(u)
\geq \Phi(0) - \inf_{u \in \Rset} f_t(u) = f_t(0) - \inf_{u \in \Rset} f_t(u)$,
\ignore{
\begin{align*}
  \sT(t)
  = \inf_{u \leq 0} f_t(u) - \inf_{u \in \Rset} f_t(u)
  \geq \Phi(0) - \inf_{u \in \Rset} f_t(u) = f_t(0) - \inf_{u \in \Rset} f_t(u),
  & \geq \inf_{u \leq 0} \paren*{\Phi(0) - t u \Phi'(0)}
  - \inf_{u \in \Rset} f_t(u)\\
  & = \Phi(0) - \inf_{u \in \Rset} f_t(u) = f_t(0) - \inf_{u \in \Rset} f_t(u) \tag{$u < 0$, $\Phi'(0) > 0$},
\end{align*}}
where equality is achieved when $u = 0$.
\end{proof}
\ignore{
As an example,
for the logistics loss, $a^*_t$ is given by $a^*_t = \log \paren*{\frac{1 + t}{1 - t}}$, and for the exponential loss by $a^*_t = \frac{1}{2} \log \paren*{\frac{1 + t}{1 - t}}$, for any $t \in [0, 1]$
}

\begin{restatable}[Upper and lower bound for binary margin-based losses]{theorem}{BinaryLower}
\label{thm:binary-lower}
Let $\sH$ be a complete hypothesis set. Assume that $\Phi$ is convex, twice continuously differentiable, and satisfies the inequalities $\Phi'(0) > 0$ and $\Phi''(0) > 0$. Then, the following property holds: $\sT(t) = \Theta (t^2)$; that is, there exist positive constants $C > 0$, $c > 0$, and $T > 0$ such that $C t^2 \geq \sT(t) \geq c t^2 $, for all $0 < t \leq T$.
\end{restatable}
\noindent \textbf{Proof sketch} First, we demonstrate that, by applying the implicit function theorem, $\inf_{u \in \Rset} f_t(u)$ is attained uniquely by $a^*_t$, and that $a^*_t$ is continuously differentiable over $[0, \e]$ for some $\e > 0$. The minimizer $a^*_t$ satisfies the following condition:
$
f'_t(a^*_t) = \frac{1 - t}{2} \Phi'(a^*_t) - \frac{1 + t}{2} \Phi'(-a^*_t) = 0.
$
Specifically, at $t = 0$, we have $\Phi'(a^*_0) =
\Phi'(-a^*_0)$. Then, by the convexity of $\Phi$ and monotonicity of
the derivative $\Phi'$, we must have $a^*_0 = 0$ and since $\Phi'$ is
non-decreasing and $\Phi''(0) > 0$, we have $a^*_t > 0$ for all $t \in
(0, \e]$. Furthermore, since $a^*_t$ is a function of class $C^1$, we
  can differentiate this condition with respect to $t$ and take the
  limit $t \to 0$, which gives the following equality: $\frac{d
    a_t^*}{d t}(0) = \frac{\Phi'(0)}{\Phi''(0)} > 0$. Since $\lim_{t
    \to 0} \frac{a_t^*}{t} = \frac{d a_t^*}{d t}(0) =
  \frac{\Phi'(0)}{\Phi''(0)} > 0$, we have $a^*_t = \Theta(t)$.  By
  Theorem~\ref{thm:binary-char} and Taylor's theorem with an integral
  remainder, $\sT$ can be expressed as follows: for any $t \in [0,
    \e]$, $\sT(t) = f_t(0) -\inf_{u \in \Rset} f_t(u) = \int_0^{a^*_t}
  u f''_t(u) \, du = \int_0^{a^*_t} u \bracket*{\frac{1 - t}{2}
    \Phi''(u) + \frac{1 + t}{2} \Phi''(-u)} \, du$. Since $\Phi''(0) >
  0$ and $\Phi''$ is continuous, there is a non-empty interval
  $[\minus \alpha, \plus \alpha]$ over which $\Phi''$ is
  positive. Since $a^*_0 = 0$ and $a^*_t$ is continuous, there exists
  a sub-interval $[0, \epsilon'] \subseteq [0, \epsilon]$ over which
  $a^*_t \leq \alpha$. Since $\Phi''$ is continuous, it admits a
  minimum and a maximum over any compact set and we can define $c =
  \min_{u \in [-\alpha, \alpha]} \Phi''(u)$ and $C = \max_{u \in
    [-\alpha, \alpha]} \Phi''(u)$. $c$ and $C$ are both positive since
  we have $\Phi''(0) > 0$. Thus, for $t$ in $[0, \epsilon']$, the
  following inequality holds:
$
C \frac{(a^*_t)^2}{2} = \int_0^{a^*_t}  u C \, du \geq \sT(t) = \int_0^{a^*_t} u \bracket*{\frac{1 - t}{2} \Phi''(u) + \frac{1 + t}{2} \Phi''(-u)} \, du
\geq \int_0^{a^*_t}  u c \, du
= c \frac{(a^*_t)^2}{2}.
$
This implies that $\sT(t) = \Theta(t^2)$. The full proof is included in Appendix~\ref{app:binary-lower}.

Theorem~\ref{thm:binary-lower} directly applies to excess error bounds
as well, when $\sH = \sH_{\rm{all}}$. Importantly, our lower bound
requires weaker conditions than
\citep[Theorem~4]{frongillo2021surrogate}, and our upper bound is
entirely novel.  This result demonstrates that the growth rate for
these loss functions is precisely square-root, refining the ``at least
square-root'' finding of these authors. It is known that polyhedral
losses admit a linear grow rate \citep{frongillo2021surrogate}. Thus,
a striking dichotomy emerges: $\sH$-consistency bounds for polyhedral
losses exhibit a linear growth rate, while they follow a square-root
rate for smooth loss functions (see Appendix~\ref{app:poly-smooth} for
a detailed comparison).

\section{Multi-class classification}
\label{sec:multi}

\ignore{Let $\sX$ be the input space and $\sY = \curl*{1, \ldots, n}$ the label space. We will specifically consider the multi-class case where $n > 2$. We denote by $\sH$ a hypothesis set. A hypothesis $h \in \sH$ is defined based on a scoring function $h \colon \sX \times \sY \to \Rset$.
Let $\ell \colon \sH \times \sX \times \sY \to \Rset$ be a loss function. We will specifically consider the target multi-class zero-one loss $\ell_{0-1}\colon (h, x, y) \mapsto 1_{\hh(x) \neq y}$, where $\hh(x) = \argmax_{y \in \sY} h(x, y)$, with ties broken by selecting the label with the highest index under the natural ordering of labels.}
In this section, we will study two families of surrogate losses in multi-class classification: comp-sum losses and constrained losses, defined in Section~\ref{sec:comp} and Section~\ref{sec:cstnd}  respectively. Comp-sum losses and constrained losses are
general and cover all loss functions commonly used in practice. We will consider any hypothesis set $\sH$ that is \emph{symmetric} and
\emph{complete}. We say that a hypothesis set is \emph{symmetric} when
it does not depend on a specific ordering of the classes, that is,
when there exists a family $\sF$ of functions $f$ mapping from $\sX$
to $\Rset$ such that $\curl*{\bracket*{h(x, 1), \ldots, h(x, n)} \colon
  h \in \sH} = \curl*{\bracket*{f_1(x), \ldots, f_n(x)} \colon f_1,
  \ldots, f_n \in \sF}$, for any $x \in \sX$. We say that a hypothesis
set $\sH$ is \emph{complete} if the set of scores it generates spans
$\Rset$, that is, $\curl*{h(x, y) \colon h \in \sH} = \Rset$, for any
$(x, y) \in \sX \times \sY$.

\ignore{Let $y_{\max}=\argmax_{y\in \sY} p_y$, where we choose the label with the highest index under the natural ordering of labels as the tie-breaking strategy, as with $\hh(x) =
\argmax_{y\in \sY}h(x, y)$.}

\subsection{Comp-sum losses}
\label{sec:comp}

Here, we consider comp-sum losses \citep{mao2023cross}, defined as
\[
\forall h \in \sH, \forall (x, y) \times \sX \times \sY,  \quad
\ell^{\rm{comp}}(h, x, y)
= \Phi \paren*{\frac{e^{ h(x, y)}}{\sum_{y'\in \sY} e^{h(x, y')}}},
\]
where $\Phi \colon \Rset \to \Rset_{+}$ is a non-increasing
function. For example, $\Phi$ can be chosen as the negative log
function $u \mapsto -\log(u)$ for the comp-sum losses, which leads to
the multinomial logistic loss.  As shown by
\citet{MaoMohriZhong2023characterization}, for symmetric and complete hypothesis
sets, the transformation $\sT$ for the family of comp-sum losses can
be characterized as follows.
\begin{theorem}[{\citet[Theorem~3]{MaoMohriZhong2023characterization}}]
\label{Thm:char_comp}
Let $\sH$ be a symmetric and complete hypothesis set.  Assume that $\Phi$ is convex,
differentiable at $\frac12$ and satisfies the inequality
$\Phi'(\frac12) < 0$. Then, the transformation $\sT$ can be expressed
as
\begin{align*}
  \sT(t) = \inf_{\tau \in \bracket*{\frac1n, \frac12}}
  \sup_{|u| \leq \tau} \curl*{\Phi(\tau) -  \frac{1 - t}{2}\Phi(\tau + u)
    - \frac{1 + t}{2} \Phi \paren*{\tau - u}}.
\end{align*}
\end{theorem}
Next, we will show that as with the binary case, for the comp-sum
losses, the properties $\sT(t) = \Omega(t^2)$ and $\sT(t) = O(t^2)$
hold.\ignore{We first introduce a corollary of
  Theorem~\ref{thm:a_implicit}, which characterizes the optimal
  solution of a class of constrained convex optimization problems.}
We first introduce a generalization of the classical implicit function
theorem where the function takes the value zero over a set of points
parameterized by a compact set. We treat the special case of a
function $F$ defined over $\Rset^3$ and denote by $(t, a, \tau)
\in \Rset^3$ its arguments. The theorem holds more generally for the
arguments being in $\Rset^{n_1} \times \Rset^{n_2} \times \Rset^{n_3}$
and with the condition on the partial derivative being non-zero
replaced with a partial Jacobian being non-singular.

\begin{theorem}[Implicit function theorem with a compact set]
\label{thm:a_implicit}
Let $F\colon \Rset \times \Rset \times \Rset \to \Rset$ be a
continuously differentiable function in a neighborhood of $(0, 0,
\tau)$, for any $\tau$ in a non-empty compact set $\sC$, with $F(0, 0,
\tau) = 0$. Then, if $\frac{\partial F}{\partial a} (0, 0, \tau)$ is
non-zero for all $\tau$ in $\sC$, then, there exist a neighborhood
$\sO$ of $0$ and a unique function $\bar a$ defined over $\sO \times
\sC$ that is continuously differentiable and satisfies
\[
\forall (t, \tau) \in \sO \times \sC,
\quad F(t, \bar a(t, \tau), \tau) = 0.
\]
\end{theorem}
\begin{proof}
  By the implicit function theorem (see for example
  \citep{DontchevRockafellar2009}), for any $\tau \in \sC$, there
  exists an open set $\sU_\tau = (-t_\tau, +t_\tau) \times (\tau -
  \e_\tau, \tau + \e_\tau)$, ($t_\tau > 0$ and $\e_\tau > 0$), and a
  unique function $\bar a_\tau \colon \sU_\tau \to \Rset$ that is in
  $C^1$ and such that for all $(t, \tau) \in \sU_\tau$, $F(t,
  \bar a_\tau(t), \tau) = 0$.

  By the uniqueness of $\bar a_\tau$, for any $\tau \neq \tau'$ and
  $(t_1, \tau_1) \in \sU_\tau \cap \sU_{\tau'}$, we have $\bar
  a_\tau(t_1, \tau_1) = \bar a_{\tau'}(t_1, \tau_1)$. Thus, we can
  define a function $\bar a$ over $\sU = \bigcup_{\tau \in \sC}
  \sU_\tau$ that is of class $C^1$ and such that for any $(t, \tau)
  \in \sU$, $F(t, \bar a(t, \tau), \tau) = 0$.

  Now, $\bigcup_{\tau \in \sC} (\tau - \e_\tau, \tau + \e_\tau)$ is a
  cover of the compact set $\sC$ via open sets. Thus, we can extract
  from it a finite cover $\bigcup_{\tau \in I} (\tau - \e_\tau, \tau +
  \e_\tau)$, for some finite cardinality set $I$. Define $(-t_0, +t_0)
  = \bigcap_{\tau \in I} (-t_\tau, +t_\tau)$, which is a non-empty
  open interval as an intersection of (embedded) open intervals
  containing zero. Then, $\bar a$ is continously differentiable over
  $(-t_0, +t_0) \times \sC$ and for any $(t, \tau) \in (-t_0, +t_0)
  \times \sC$, we have $F(t, \bar a(t, \tau), \tau) = 0$.
\end{proof}

\ignore{
The proof will make use of the following lemma.
\begin{lemma}
  \label{lemma:cont}
  Let $\Phi$ be a non-negative and continuous function defined over
  $\Rset$. Then, for any $t \in [0, 1]$, the function $g$ defined over
  $\Rset_+$ by
  \[
g(\tau) =  \sup_{|u| \leq \tau} \curl*{
  \frac{1 - t}{2} \Phi(\tau + u) + \frac{1 + t}{2} \Phi\paren*{\tau - u}}
\]
is continuous.
  
\end{lemma}
\begin{proof}
  Fix $t \in [0, 1]$. For any $\tau \geq 0$, $g(\tau)$ can be equivalently
  expressed as follows:
\begin{align*}
  g(\tau)
  & =  \sup_{|u| \leq \tau} \curl*{
    \frac{1 - t}{2} \Phi(\tau + u) + \frac{1 + t}{2} \Phi\paren*{\tau - u}}\\
  & =  \sup_{|u| \leq 1} \curl*{
    \frac{1 - t}{2} \Phi\paren*{(1 + u) \tau} + \frac{1 + t}{2} \Phi\paren*{(1 - u)\tau}}.
\end{align*}
Since $\Phi$ is continuous, for any $u \in [-1, +1]$,
$\tau \mapsto \frac{1 - t}{2} \Phi((1 + u) \tau) + \frac{1 + t}{2} \Phi((1 - u)\tau)$ is continuous. Since the supremum over a fixed compact
set of a family of continuous functions is continuous, this shows that
that $g$ is a continuous function.
\end{proof}
}

\begin{restatable}[Upper and lower bound for comp-sum losses]{theorem}{CompLower}
\label{thm:comp-lower}
Assume that $\Phi$ is convex, twice continuously differentiable, and satisfies the properties $\Phi'(u) < 0$ and $\Phi''(u) > 0$ for any $u \in (0, \frac12]$.
Then, the following property holds: 
$\sT(t) = \Theta(t^2)$.
\end{restatable}
\noindent \textbf{Proof sketch}
For any $\tau \in \bracket*{\frac1n, \frac12}$, define the function $\sT_\tau$ by
$ \sT_\tau(t) = f_{t, \tau}(0) - \inf_{|u| \leq \tau} f_{t, \tau}(u),
$
where
$
f_{t, \tau}(u)
= \frac{1 - t}{2} \Phi_{\tau}(u) + \frac{1 + t}{2} \Phi_{\tau}(-u)$, $t \in [0, 1]$ and $
\Phi_{\tau}(u) = \Phi(\tau + u).
$

We aim to establish a lower and upper bound for $\inf_{\tau \in
  \bracket*{\frac1n, \frac12}} \sT_\tau(t)$.  For any fixed $\tau \in
\bracket*{\frac1n, \frac12}$, this situation is parallel to that of
binary classification (Theorem~\ref{thm:binary-char} and
Theorem~\ref{thm:binary-lower}), since we have $\Phi'_{\tau}(0) =
\Phi'(\tau) < 0$ and $\Phi''_{\tau}(0) = \Phi''(\tau) > 0$.  By
Theorem~\ref{thm:a_implicit} and the proof of
Theorem~\ref{thm:binary-lower}, adopting a similar notation, while
incorporating the $\tau$ subscript to distinguish different functions
$\Phi_\tau$ and $f_{t, \tau}$, we can write $ \forall t \in [0,
  t_0],\, \sT_\tau(t) = \int_0^{-a^*_{t, \tau}} u \bracket*{\frac{1 -
    t}{2} \Phi''_{\tau}(-u) + \frac{1 + t}{2} \Phi''_{\tau}(u)} \, du
$, where $a^*_{t, \tau}$ verifies $ a_{0, \tau}^* = 0$ and $
\frac{\partial a_{t, \tau}^*}{\partial t}(0) =
\frac{\Phi'_{\tau}(0)}{\Phi''_{\tau}(0)} = c_\tau < 0.  $ Then, by
further analyzing this equality, we can show the lower bound
$\inf_{\tau \in \bracket*{\frac1n, \frac12}} -a_{t, \tau}^* =
\Omega(t)$ and the upper bound $\sup_{\tau \in \bracket*{\frac1n,
    \frac12}} -a_{t, \tau}^* = O(t)$ for some $t \in [0, t_1]$, $t_1 >
0$. Finally, using the fact that $\Phi''$ reaches its maximum and
minimum over a compact set, we obtain that $ \sT(t) = \inf_{\tau \in
  \bracket*{\frac1n, \frac12}} \sT_\tau(t) = \Theta(t^2)$. The full
proof is included in Appendix~\ref{app:comp-lower}.

Theorem~\ref{thm:comp-lower} significantly extends
Theorem~\ref{thm:binary-lower} to multi-class comp-sum losses, which
include the logistic loss or cross-entropy used with a softmax
activation function. It shows that the growth rate of
$\sH$-consistency bounds for comp-sum losses is exactly square-root,
provided that the auxiliary function $\Phi$ they are based upon is
convex, twice continuously differentiable, and satisfies $\Phi'(u) <
0$ and $\Phi''(u) > 0$ for any $u \in (0, \frac{1}{2}]$, which holds
for most loss functions used in practice.

\subsection{Constrained losses}
\label{sec:cstnd}

Here, we consider constrained losses (see
\citep{lee2004multicategory}), defined as
\[
\forall h \in \sH, \forall (x, y) \times \sX \times \sY,  \quad
\ell^{\mathrm{cstnd}}(h, x, y)
= \sum_{y'\neq y}\Phi\paren*{h(x, y')} \text{ subject to }
\sum_{y\in \sY} h(x, y) = 0,
\]
where $\Phi \colon \Rset \to \Rset_{+}$ is a non-decreasing
function. On possible choice for $\Phi$ is the exponential function.
As shown by \citet{MaoMohriZhong2023characterization}, for symmetric and complete
hypothesis sets, the transformation $\sT$ for the family of
constrained losses can be characterized as follows.

\begin{theorem}[{\citet[Theorem~11]{MaoMohriZhong2023characterization}}]
\label{Thm:char_cstnd}
Let $\sH$ be a symmetric and complete hypothesis set.  Assume that $\Phi$ is convex,
differentiable at zero and satisfies the inequality $\Phi'(0) >
0$. Then, the transformation $\sT$ can be expressed as
\begin{align*}
\sT(t) =
\inf_{\tau \geq 0}\sup_{u \in \Rset}
\curl*{\paren[\Big]{2 - \frac{1}{n-1}} \Phi(\tau)
  - \frac{2 - \frac{1}{n-1} - t}{2} \Phi(\tau + u)
  - \frac{2 - \frac{1}{n-1} + t}{2}\Phi(\tau - u)}.
\end{align*}
\end{theorem}
Next, we will show that for the constrained losses, the properties
$\sT(t) = \Omega(t)$ and $\sT(t) = O(t)$ hold as well. Note that by
Theorem~\ref{Thm:char_cstnd}, we have
\begin{equation*}
  \sT\paren*{\paren[\Big]{2 - \frac{1}{n-1}}t}
  = \paren*{2 - \frac{1}{n-1}} \inf_{\tau \geq 0} \sup_{u \in \Rset}
  \curl*{\Phi(\tau) - \frac{1 - t}{2} \Phi(\tau + u) - \frac{1 + t}{2}\Phi(\tau - u)}.
\end{equation*}
Therefore, to prove $\sT(t) = \Theta(t^2)$, we only need to show 
\begin{align*}
\inf_{\tau \geq 0} \sup_{u \in \Rset} \curl*{\Phi(\tau) - \frac{1 - t}{2} \Phi(\tau + u) - \frac{1 + t}{2}\Phi(\tau - u)} &= \Theta(t^2).
\end{align*}
For simplicity, we assume that the infimum over $\tau \geq 0$ can be
reached within some finite interval $[0, A]$, $A > 0$. This assumption
holds for common choices of $\Phi$, as discussed in
\citep{MaoMohriZhong2023characterization}. Furthermore, as
demonstrated in Appendix~\ref{app:analysis}, under certain conditions
on $\Phi''$, the infimum over $\tau \in [0, A]$ is reached at zero for
sufficiently small values of $t$. For specific examples, see
\citep[Appendix D.3]{MaoMohriZhong2023characterization}, where
$\Phi(t) = e^{t}$ is considered.

\begin{restatable}[Upper and lower bound for constrained losses]{theorem}{CstndLower}
\label{thm:cstnd-lower}
Assume that $\Phi$ is convex, twice continuously differentiable, and
satisfies the properties $\Phi'(u) > 0$ and $\Phi''(u) > 0$ for any $u
\geq 0$.  Then, for any $A > 0$, the following property holds:
\[
\inf_{\tau \in [0, A]} \sup_{u \in \Rset}
\curl*{\Phi(\tau) - \frac{1 - t}{2} \Phi(\tau + u) - \frac{1 + t}{2}\Phi(\tau - u) }
= \Theta(t^2).
\]
\end{restatable}
\noindent \textbf{Proof sketch} For any $\tau \in [0, A]$, define the
function $\sT_\tau$ by $ \sT_\tau(t) = f_{t, \tau}(0) - \inf_{u \in
  \Rset} f_{t, \tau}(u)$, where $f_{t, \tau}(u) = \frac{1 - t}{2}
\Phi_{\tau}(u) + \frac{1 + t}{2} \Phi_{\tau}(-u)$, $t \in [0, 1]$ and
$\Phi_{\tau}(u) = \Phi(\tau + u)$.  We aim to establish a lower and
upper bound for $\inf_{\tau \in [0, A]} \sT_\tau(t)$.  For any fixed
$\tau \in [0, A]$, this situation is parallel to that of binary
classification (Theorem~\ref{thm:binary-char} and
Theorem~\ref{thm:binary-lower}), since we also have $\Phi'_{\tau}(0) =
\Phi'(\tau) > 0$ and $\Phi''_{\tau}(0) = \Phi''(\tau) > 0$.  By
applying Theorem~\ref{thm:a_implicit} and leveraging the proof of
Theorem~\ref{thm:binary-lower}, adopting a similar notation, while
incorporating the $\tau$ subscript to distinguish different functions
$\Phi_\tau$ and $f_{t, \tau}$, we can write $ \forall t \in [0, t_0],
\, \sT_\tau(t) = \int_0^{a^*_{t, \tau}} u \bracket*{\frac{1 - t}{2}
  \Phi''_{\tau}(u) + \frac{1 + t}{2} \Phi''_{\tau}(-u)} \, du, $ where
$a^*_{t, \tau}$ verifies $ a_{0, \tau}^* = 0$ and $ \frac{\partial
  a_{t, \tau}^*}{\partial t}(0) =
\frac{\Phi'_{\tau}(0)}{\Phi''_{\tau}(0)} = c_\tau > 0.  $ Then, by
further analyzing this equality, we can show the lower bound
$\inf_{\tau \in [0, A]} a_{t, \tau}^* = \Omega(t)$ and the upper bound
$\sup_{\tau \in [0, A]} a_{t, \tau}^* = O(t)$ for some $t \in [0,
  t_1]$, $t_1 > 0$. Finally, using the fact that $\Phi''$ reaches its
maximum and minimum over some compact set, we obtain that $ \sT(t) =
\inf_{\tau \in [0, A]} \sT_\tau(t) = \Theta(t^2)$. The full proof is
included in Appendix~\ref{app:cstnd-lower}.

Theorem~\ref{thm:cstnd-lower} substantially expands our findings to
multi-class constrained losses.  It demonstrates that, under some
assumptions, which are commonly satisfied by smooth constrained losses
used in practice, constrained loss $\sH$-consistency bounds also
exhibit a square-root growth rate.

\section{Minimizability gaps}
\label{sec:M-gaps}

As shown in Sections~\ref{sec:binary} and \ref{sec:multi}, the
$\sH$-consistency bounds of smooth loss functions in both binary and
multi-class classification all have a square-root growth rate near
zero. Next, we start by examining how the number of classes affects these
bounds. Then, we focus on the minimizability gaps, which are the only
factors that differ between the bounds.

\subsection{Dependency on number of classes}
  
Even with identical growth rates, surrogate losses can vary in their
$\sH$-consistency bounds due to the number of classes. This factor
becomes crucial to consider when the class count is large.
Consider the family of comp-sum loss functions
$\ell_{\tau}^{\rm{comp}}$ with $\tau
\in [0, 2)$, defined as 
\begin{equation*}
\ell_{\tau}^{\rm{comp}}(h, x, y)
= \Phi^{\tau} \paren*{ \frac{ e^{ h(x, y)}}{\sum_{y'\in \sY} e^{h(x, y') } } } 
=
\begin{cases}
  \frac{1}{1 - \tau}
  \paren*{\bracket*{\sum_{y'\in\sY} e^{{h(x, y') - h(x, y)}}}^{1 - \tau} - 1}
  & \tau \neq 1,  \tau \in [0, 2) \\
\log\paren*{\sum_{y'\in \sY} e^{h(x, y') - h(x, y)}} & \tau = 1,
\end{cases}
\end{equation*}
where $\Phi^{\tau}(u) = -\log (u) 1_{\tau = 1} + \frac{1}{1 - \tau}
\paren*{u^{\tau - 1} - 1} 1_{\tau \neq 1}$, for any $\tau \in [0,
  2)$. \citet[Eq.~(7) \& Theorem~3.1]{mao2023cross}, established
  the following bound for any $h \in \sH$ and $\tau \in [1, 2)$,
\begin{align*}
\sR_{\ell_{0-1}}(h) - \sR_{\ell_{0-1}}^*(\sH)
\leq \wt \Gamma_{\tau}
  \paren*{\sR_{\ell_{\tau}^{\rm{comp}}}(h) - \sR_{\ell_{\tau}^{\rm{comp}}}^*(\sH)
    + \sM_{\ell_{\tau}^{\rm{comp}}}(\sH)}
- \sM_{\ell_{0-1}}(\sH),
\end{align*}
where $\wt \Gamma_{\tau}(t) = \sqrt{2n^{\tau-1} t}$.  Thus, while all
these loss functions show square-root growth, the number of classes
acts as a critical scaling factor.

\subsection{Comparison of minimizability gaps and small surrogate minimizability gaps}
In Appendix~\ref{app:M-gaps-comp}, we compare minimizability gaps cross comp-sum losses.
We will see that minimizability gaps decrease as $\tau$ increases. This might suggest
favoring $\tau$ close to $2$.  But when accounting
for $n$, $\ell_{\tau}^{\rm{comp}}$ with $\tau = 1$ (logistic loss) is
optimal since $n$ then vanishes. Thus, both class count and
minimizability gaps are essential in loss selection. In Appendix~\ref{app:small-M-gaps}, we will show that the minimizability gaps can become
zero or relatively small under certain conditions.  In such scenarios,
$\ell_{\tau}^{\rm{comp}}$ with $\tau = 1$ (logistic loss) is favored,
which can partly explain its widespread practical application.

\section{Conclusion}

We established a universal square-root growth rate for the widely-used
class of smooth surrogate losses in both binary and multi-class
classification. This underscores the minimizability gap as a
crucial discriminator among surrogate losses. Our detailed analysis of
these gaps can provide guidance for loss selection.

\ignore{
We showed a universal square-root growth rate for smooth surrogate
losses, the predominant choice in neural network training (over
polyhedral losses) for both binary and multi-class classification in
applications. We demonstrated that given this universal growth rate,
the minimizability gap is the key factor for differentiating the
surrogate losses. Thus, we provided a detailed analysis of
minimizability gaps, offering practical guidance for surrogate loss
selection. Our analysis can be similarly applied to other scenarios
and benefit real-world applications.
}

% \acks{}

\bibliography{srd}

\newpage
\appendix

\renewcommand{\contentsname}{Contents of Appendix}
\tableofcontents
\addtocontents{toc}{\protect\setcounter{tocdepth}{4}} 
\clearpage

\section{Related work}
\label{app:related}

The Bayes-consistency of surrogate losses has been extensively studied
in the context of binary classification.  \citet{Zhang2003},
\citet{bartlett2006convexity} and \citet{steinwart2007compare}
established Bayes-consistency for various convex loss functions,
including margin-based surrogates. They also introduced excess error
bounds (or surrogate regret bounds) for margin-based surrogates.
\citet{reid2009surrogate} extended these results to proper losses in
binary classification.

The Bayes-consistency of several surrogate loss function families in
the context of multi-class classification has also been studied
by \citet{zhang2004statistical} and
\citet{tewari2007consistency}. \cite{zhang2004statistical} established
a series of results for various multi-class classification
formulations, including negative results for multi-class hinge loss
functions \citep{crammer2001algorithmic}, as well as positive results
for the sum exponential loss
\citep{WestonWatkins1999,AwasthiMaoMohriZhong2022multi}, the
(multinomial) logistic loss
\citep{Verhulst1838,Verhulst1845,Berkson1944,Berkson1951}, and the
constrained losses \citep{lee2004multicategory}. Later,
\citet{tewari2007consistency} adopted a different geometric method to
analyze Bayes-consistency, yielding similar results for these loss
function families.  \citet{steinwart2007compare} developed general
tools to characterize Bayes consistency for both binary and
multi-class classification. Additionally, excess error bounds have
been derived by \cite{AvilaPiresSzepesvariGhavamzadeh2013} for a
family of constrained losses and by \citet{DuchieKhosraviRuan2018} for
loss functions related to generalized entropies.

For a surrogate loss $\ell$, an excess error bound holds for any
predictor $h$ and has the form $\sE_{\ell_{0-1}}\! (h) -
\sE_{\ell_{0-1}}^*\!\! \leq \Psi(\sE_{\ell}(h) - \sE_{\ell}^*)$, where
$\sE_{\ell_{0-1}}\!(h)$ and $\sE_{\ell}(h)$ represent the expected
losses of $h$ for the zero-one loss and surrogate loss respectively,
$\sE_{\ell_{0-1}}^*\!$ and $\sE_{\ell}^*$ the Bayes errors for the
zero-one and surrogate loss respectively, and $\Psi$ a non-decreasing
function.

The \emph{growth rate} of excess error bounds, that is the behavior of
function $\Psi$ near zero, has gained attention in recent research
\citep{mahdavi2014binary,zhang2021rates,
  frongillo2021surrogate,bao2023proper}.  \cite{mahdavi2014binary}
examined the growth rate for \emph{smoothed hinge losses} in binary
classification, demonstrating that smoother losses result in worse
growth rates. The optimal rate is achieved with the standard hinge
loss, which exhibits linear growth.  \cite{zhang2021rates} tied the
growth rate of excess error bounds in binary classification to two
properties of the surrogate loss function: consistency intensity and
conductivity.  These metrics enable comparisons of growth rates across
different surrogates. But, can we establish lower and upper bounds
for the growth rate of excess error bounds under specific regularity
conditions?

\citet{frongillo2021surrogate} pioneered research on this question in
binary classification settings.  They established a critical
square-root lower bound for excess error bounds when a surrogate loss
is locally strongly convex and has a locally Lipschitz gradient.
Additionally, they demonstrated a linear excess error bound for
Bayes-consistent polyhedral loss functions (convex and
piecewise-linear) \citep{finocchiaro2019embedding} (see also
\citep{lapin2016loss,ramaswamy2018consistent,yu2018lovasz,
  yang2020consistency}).  More recently, \citet{bao2023proper}
complemented these results by showing that proper losses associated
with Shannon entropy, exponential entropy, spherical entropy, squared
$\alpha$-norm entropies and $\alpha$-polynomial entropies, with
$\alpha > 1$, also exhibit a square-root lower bound for excess error
bounds relative to the $\ell_1$-distance.

However, while Bayes-consistency and excess error bounds are valuable,
they are not sufficiently informative, as they are established for the
family of all measurable functions and disregard the crucial role
played by restricted hypothesis sets in learning. As pointed out by
\citet{long2013consistency}, in some cases, minimizing
Bayes-consistent losses can result in constant expected error, while
minimizing inconsistent losses can yield an expected loss approaching
zero.  To address this limitation, the authors introduced the concept
of \emph{realizable $\sH$-consistency}, further explored by
\citet{KuznetsovMohriSyed2014} and \citet{zhang2020bayes}.
Nonetheless, these guarantees are only asymptotic and rely on a strong
realizability assumption that typically does not hold in practice.

Recent research by \citet{awasthi2022h,AwasthiMaoMohriZhong2022multi}
and \citet{mao2023cross,MaoMohriZhong2023ranking,
  MaoMohriZhong2023structured,MaoMohriZhong2023characterization} has
instead introduced and analyzed \emph{$\sH$-consistency bounds}.
These bounds are more informative than Bayes-consistency since they
are hypothesis set-specific and non-asymptotic.  Their work covers
broad families of surrogate losses in binary classication, multi-class
classification, structured prediction, and abstention
\citep{MaoMohriMohriZhong2023twostage}.  Crucially, they provide upper
bounds on the \emph{estimation error} of the target loss, for example,
the zero-one loss in classification, that hold for any predictor $h$
within a hypothesis set $\sH$. These bounds relate this estimation
error to the surrogate loss estimation error.  Their general form is:
$\sE_{\ell_{0-1}}\!(h) - \sE^*_{\ell_{0-1}}\!(\sH) \leq
f \paren*{\sE_{\ell}(h) - \sE^*_{\ell}(\sH)}$, where
$\sE^*_{\ell_{0-1}}\!(\sH)$ and $\sE^*_{\ell}(\sH)$ represent the
best-in-class expected losses for the zero-one and surrogate loss
respectively, and $f$ is a non-decreasing function continuous at
zero. $\sH$-consistency bounds imply in particular excess error
bounds, when the hypothesis set is taken to be the family of all
measurable functions.

%%%%% Move to appendix %%%%%%%%%%%%%%%%%%%%%%%%%%%%%%%%%%%%%%%%%
The authors have further analyzed $\sH$-consistency bounds in
structured prediction, ranking, and abstention.

\citet{MaoMohriZhong2023structured} revealed limitations of existing
structured prediction loss functions, demonstrating they lack
Bayes-consistency. They introduced new surrogate loss families proven
to benefit form $\sH$-consistency bounds, thus also establishing
Bayes-consistency. This complements earlier negative finding about the
Bayes-consistency of Struct-SVM and positive results for quadratic
surrogate (QS) losses or some non-smooth polyhedral-type loss
functions \citep{osokin2017structured,ciliberto2016consistent,
  blondel2019structured,
  nowak2019general,nowak2019sharp,nowak2020consistent,
  ciliberto2020general, nowak2022consistency}.

\citet{MaoMohriZhong2023ranking} showed that there are no meaningful
$\sH$-consistency bounds for general pairwise ranking and bipartite
ranking surrogate losses with equicontinuous hypothesis sets,
including linear models and neural networks.  They proposed ranking
with abstention \citep{MaoMohriZhong2023rankingabs} as a solution. These results demonstrated that although these surrogate loss functions have been shown to be Bayes-consistent in various studies \citep{kotlowski2011bipartite, menon2014bayes, agarwal2014surrogate, gao2015consistency, uematsu2017theoretically}, they are, in fact, not $\sH$-consistent.

\citet{MaoMohriMohriZhong2023twostage} applied $\sH$-consistency
bounds to two-stage learning to defer scenarios, designing new
surrogate losses. This complemented the
Bayes-consistent surrogate losses in the single-stage scenario of
learning to defer \citep{mozannar2020consistent,verma2022calibrated,
  verma2023learning,pmlr-v206-mozannar23a} or learning with abstention
\citep{bartlett2008classification,yuan2010classification,
  CortesDeSalvoMohri2016,CortesDeSalvoMohri2016bis,
  ramaswamy2018consistent, NiCHS19,charoenphakdee2021classification,
  caogeneralizing,CortesDeSalvoMohri2023}.
  
  Moreover, $\sH$-consistency bounds have also been studied in the scenario of adversarial robustness \citep{AwasthiMaoMohriZhong2023theoretically,awasthi2023dc}, the bounded regression \citep{mao2024regression,mao2024h}, the top-$k$ classification \citep{mao2024top},  the score-based abstention \citep{MaoMohriZhong2023score}, the predictor-rejector abstention \citep{MaoMohriZhong2023predictor}, learning to abstain with a fixed predictor with application in
decontextualization \citep{MohriAndorChoiCollinsMaoZhong2023learning}, and single-stage learning to defer with multiple experts \citep{MaoMohriZhong2023deferral}.

%%%%%%%%%%%%%%%%%%%%%%%%%%%%%%%%%%%%%%%%%%%%%%%%%%%%%%%%%%%%%%%%

This papers presents a characterization of the growth rate of
$\sH$-consistency bounds, that is how quickly the functions $f$
increase near zero, both in binary and multi-class classification.

\section{Proof of Lemma~\ref{lemma:approximation-error}}
\label{app:lemma}

\ApproximationError*
\begin{proof}
By definition, for a pointwise loss function $\ell$, there exists a
measurable function $\hat \ell \colon \Rset^n \times \sY \to
\Rset_{+}$ such that $\ell(h, x, y) = \hat \ell(h(x), y)$, where $h(x)
= \bracket*{h(x, 1), \ldots, h(x, n)}$ is the score vector of the
predictor $h$. Thus, the following inequality holds:
\begin{equation*}
\sC^*_{\ell}\paren*{\sH_{\rm{all}}, x} = \inf_{h \in \sH} \E_y \bracket*{\ell(h, x, y) \mid x} = \inf_{\alpha \in \Rset^n}  \E_{y} \bracket*{\hat \ell(\alpha, y) \mid x}.
\end{equation*}
Since $\hat \ell \colon (\alpha, y) \mapsto \Rset_{+}$ is measurable, the function $(\alpha, x) \mapsto \E_{y} \bracket*{\hat \ell(\alpha, y) \mid x}$ is also measurable. 
We now show that the function $x \mapsto \sC^*_{\ell}\paren*{\sH_{\rm{all}}, x} = \inf_{\alpha \in \Rset}  \E_{y} \bracket*{\hat \ell(\alpha, y) \mid x}$ is also measurable. 

To do this, we consider for any $\beta > 0$, the set $\curl*{x \colon \inf_{\alpha \in \Rset}  \E_{y} \bracket*{\hat \ell(\alpha, y) \mid x} < \beta}$ which can be expressed as 
\begin{align*}
 \curl*{x \colon \inf_{\alpha \in \Rset^n}  \E_{y} \bracket*{\hat \ell(\alpha, y) \mid x} < \beta} 
 & = \curl*{x \colon \exists \alpha \in \Rset^n \text{ such that } \E_{y} \bracket*{\hat \ell(\alpha, y) \mid x} < \beta }\\
 & = \Pi_{\sX} \curl*{(\alpha, x)\colon \E_{y} \bracket*{\hat \ell(\alpha, y) \mid x} < \beta} 
\end{align*}
where $\Pi_{\sX}$ is the projection onto $\sX$. By the measurable projection theorem, $x \mapsto \inf_{\alpha \in \Rset^n}  \E_{y} \bracket*{\hat \ell(\alpha, y) \mid x}$ is measurable. Then, since a pointwise difference of measurable functions is measurable, for all $n \in \Nset$, the set $\curl*{(\alpha, x) \colon \E_{y} \bracket*{\hat \ell(\alpha, y) \mid x} < \inf_{\alpha \in \Rset^n}  \E_{y} \bracket*{\hat \ell(\alpha, y) \mid x} + \frac1n}$ is measurable. Thus, by  \citet{KuratowskiRyllNardzewski1965}'s measurable selection theorem, for all $n \in \Nset$, there exists a measurable function $h_{n} \colon x \mapsto \alpha \in \Rset^n$ such that the following holds:
\begin{equation*}
\sC_{\ell}\paren*{h_n, x}  = \E_{y} \bracket*{\hat \ell(\alpha, y) \mid x} < \inf_{\alpha \in \Rset^n}  \E_{y} \bracket*{\hat \ell(\alpha, y) \mid x} + \frac1n =  \sC^*_{\ell}\paren*{\sH_{\rm{all}}, x} + \frac1n.
\end{equation*}
Therefore, we have
\begin{equation*}
\sE^*_{\ell}(\sH_{\rm{all}}) \leq \E_{x}\bracket*{\sC_{\ell}\paren*{h_n, x}}\leq \E_{x} \bracket*{\sC^*_{\ell}\paren*{\sH_{\rm{all}}, x}} + \frac1n \leq \sE^*_{\ell}(\sH_{\rm{all}}) + \frac1n.
\end{equation*}
By taking the limit $n \to \plus \infty$, we obtain $\sE^*_{\ell}(\sH_{\rm{all}})  = \E_{x} \bracket*{\sC^*_{\ell}\paren*{\sH_{\rm{all}}, x}}$. By definition, $\sA_{\ell}(\sH) = \sE_{\ell}^*(\sH) -
  \sE_{\ell}^*\paren*{\sH_{\mathrm{all}}} = \sE_{\ell}^*(\sH) - \E_{x} \bracket*{\sC^*_{\ell}\paren*{\sH_{\rm{all}}, x}}$. This completes the proof.
\end{proof}

\section{General form of \texorpdfstring{$\sH$}{H}-consistency bounds}
\label{app:explicit-form}

Fix a target loss function $\ell_2$ and a surrogate loss $\ell_1$.
Given a hypothesis set $\sH$, a bound relating the
estimation errors of these loss functions admits the following form:
\begin{align}
\label{eq:bound}
\forall h \in \sH, \quad
	\sE_{\ell_2}(h) - \sE^*_{\ell_2}(\sH)
    	\leq \Gamma_{\sD} \paren*{\sE_{\ell_1}(h) - \sE^*_{\ell_1}(\sH)},
\end{align}
where, for any distribution $\sD$, $\Gamma_{\sD}\colon \Rset_+ \to
\Rset_+$ is a non-decreasing function on $\Rset_+$. We will assume
that $\Gamma_{\sD}$ is concave.
In particular, the bound should hold for any point mass distribution
$\delta_x$, $x \in \sX$. We will operate under the assumption that the
same bound holds uniformly over $\sX$ and thus that there exists a
fixed concave function $\Gamma$ such that $\Gamma_{\delta_x} = \Gamma$ for
all $x$.

Observe that for any point mass distribution $\delta_x$, the
conditional loss and the expected loss coincide and therefore that we
have $\sE_{\ell_2}(h) - \sE^*_{\ell_2}(\sH) = \Delta \sC_{\ell_2, \sH}(\sH,
x)$, and similarly with $\ell_1$. Thus, we can write:
\[
\forall h \in \sH, \forall x \in \sX, \quad
	\Delta \sC_{\ell_2, \sH}(h, x)
	\leq \Gamma\paren*{\Delta \sC_{\ell_1, \sH}(h, x)}.
\]
Therefore, by Jensen's inequality, for any distribution $\sD$, we have
\[
\forall h \in \sH, \forall x \in \sX, \quad
	\E_{x}\bracket*{\Delta \sC_{\ell_2, \sH}(h, x)}
	\leq \E_{x}\bracket*{\Gamma\paren*{\Delta	\sC_{\ell_1, \sH}(h, x)}}
	\leq \Gamma\paren*{\E_{x}\bracket*{\Delta	\sC_{\ell_1, \sH}(h, x)}}.
\]
Since $\E_{x}\bracket*{\Delta \sC_{\ell_2, \sH}(h, x)} = \sE_{\ell_2}(h) -
\sE^*_{\ell_2}(\sH) + \sM_{\ell_2}(\sH)$ and similarly with $\ell_1$,
we obtain the following bound for all distributions $\sD$:
\begin{equation}
\label{eq:H-consistency-bound}
\forall h \in \sH, \quad
\sE_{\ell_2}(h) - \sE^*_{\ell_2}(\sH) + \sM_{\ell_2}(\sH)
\leq \Gamma\paren*{\sE_{\ell_1}(h)-\sE^*_{\ell_1}(\sH) + \sM_{\ell_1}(\sH)}.
\end{equation}
This leads to the general form of $\sH$-consistency bounds that we will be 
considering, which includes the key role of the minimizability gaps.

\section{Properties of minimizability gaps}
\label{app:properties}

By Lemma~\ref{lemma:approximation-error}, for a pointwise loss
function, we have $\sE^*_\ell\paren*{\sH_{\rm{all}}} =
\E_{x}\bracket*{\sC_{\ell}^*(\sH_{\rm{all}},x)}$, thus the
minimizability gap vanishes for the family of all measurable
functions.

\begin{lemma}
\label{lemma:zero-minimizability}
  Let $\ell$ be a pointwise loss function. Then, we have
  $\sM_\ell(\sH_{\rm{all}}) = 0$.
\end{lemma}
\ignore{
\begin{proof}
By definition, $\sM_{\ell}(\sH_{\rm{all}}) =
\sE^*_\ell\paren*{\sH_{\rm{all}}} -
\E_{x}\bracket*{\sC_{\ell}^*(\sH_{\rm{all}},x)}$. By
lemma~\ref{lemma:approximation-error},
$\sE^*_\ell\paren*{\sH_{\rm{all}}} =
\E_{x}\bracket*{\sC_{\ell}^*(\sH_{\rm{all}},x)}$, which
implies that $\sM_{\ell}(\sH_{\rm{all}})=0$.
\end{proof}
}
Thus, in that case, \eqref{eq:H-consistency-bound} takes the following
simpler form:
\begin{equation}
\forall h \in \sH, \quad
\sE_{\ell_2}(h) - \sE^*_{\ell_2}(\sH_{\rm{all}})
\leq \Gamma \paren*{\sE_{\ell_1}(h) - \sE^*_{\ell_1}(\sH_{\rm{all}})}.
\end{equation}
In general, however, the minimizabiliy gap is non-zero for a
restricted hypothesis set $\sH$ and is therefore important to analyze.
Let $\sI_\ell(\sH)$ be the difference of pointwise infima
$\sI_\ell(\sH) = \E_x \bracket[big]{\sC^*_\ell(\sH, x) -
  \sC^*_\ell(\sH_{\rm{all}}, x)}$, which is non-negative. Note that,
for a pointwise loss function, the minimizability gap can be
decomposed as follows in terms of the approximation error and the
difference of pointwise infima:
\begin{align*}
  \sM_\ell(\sH)
  & = \sE^*_\ell(\sH) - \sE^*_\ell\paren*{\sH_{\rm{all}}}
  + \sE^*_\ell\paren*{\sH_{\rm{all}}} - \E_x \bracket*{\sC^*_\ell(\sH, x)}\\
  & = \sA_\ell(\sH) + \sE^*_\ell\paren*{\sH_{\rm{all}}}
  - \E_x \bracket*{\sC^*_\ell(\sH, x)}\\
  & = \sA_\ell(\sH) - \sI_\ell(\sH)
  \leq \sA_\ell(\sH).
\end{align*}
Thus, the minimizabiliy gap can be upper bounded by the approximation
error. It is however a finer quantity than the approximation error and
can lead to more favorable guarantees. When the difference of
pointwise infima can be evaluated or bounded, this decomposition can
provide a convenient way to analyze the minimizability gap in terms of
the approximation error.

Note that $\sI_\ell(\sH)$ can be non-zero for families of bounded
functions.  Let $\sY = \curl*{-1, +1}$ and $\sH$ be a family of functions $h$ with
$\abs*{h(x)} \leq \Lambda$ for all $x \in \sX$ and such that all
values in $[-\Lambda, +\Lambda]$ can be reached.  Consider for example
the exponential-based margin loss: $\ell(h, x, y) = e^{-yh(x)}$. Let
$\eta(x) = p(x, +1) = \sD(Y = + 1\!\mid\! X = x)$.  Thus,
$\sC_{\ell}(h, x) = \eta(x) e^{-h(x)} + (1 - \eta(x)) e^{h(x)}$. Then,
it is not hard to see that $\sC^*_{\ell}(\sH_{\rm{all}}, x) =
2\sqrt{\eta(x)(1 - \eta(x))}$ for all $x$ but $\sC^*_{\ell}(\sH, x)$
depends on $\Lambda$ with the minimizing value for $h(x)$ being: $\min
\curl*{\frac{1}{2} \log \frac{\eta(x)}{1 - \eta(x)}, \Lambda}$ if
$\eta(x) \geq 1/2$, $\max \curl*{\frac{1}{2} \log \frac{\eta(x)}{1 -
    \eta(x)}, -\Lambda}$ otherwise. Thus, in the deterministic case,
$\sI_\ell(\sH) = e^{-\Lambda}$.

When the best-in-class error coincides with the Bayes error,
$\sE^*_\ell\paren*{\sH} = \sE^*_\ell\paren*{\sH_{\rm{all}}}$, both the
approximation error and minimizability gaps vanish.
\begin{lemma}
\label{lemma:vanish-minimizability}
For any loss function $\ell$ such that $\sE^*_\ell\paren*{\sH} =
\sE^*_\ell\paren*{\sH_{\rm{all}}} =
\E_{x}\bracket*{\sC_{\ell}^*(\sH_{\rm{all}},x)}$, we have
$\sM_\ell(\sH) = \sA_\ell(\sH) = 0$.
\end{lemma}
\begin{proof}
By definition, $\sA_{\ell}(\sH) = \sE^*_\ell(\sH) -
\sE^*_\ell\paren*{\sH_{\rm{all}}} = 0$. Since we have $\sM_\ell(\sH) \leq
\sA_\ell(\sH)$, this implies  $\sM_{\ell}(\sH)=0$.
\end{proof}

\section{Examples of \texorpdfstring{$\sH$}{H}-consistency bounds}
\label{app:bounds-example}

Here, we compile some common examples of $\sH$-consistency bounds
for both binary and multi-class classification.
Table~\ref{tab:bounds-binary}, \ref{tab:bounds-comp} and
\ref{tab:bounds-cstnd} include the examples of $\sH$-consistency
bounds for binary margin-based losses, comp-sum losses and constrained
losses, respectively.

These bounds are due to previous work by \citet{awasthi2022h} for
binary margin-based losses, by \citet{mao2023cross} for multi-class
comp-sum losses, and by \citet{AwasthiMaoMohriZhong2022multi} and
\citet{MaoMohriZhong2023characterization} for multi-class constrained
losses, respectively. We consider complete hypothesis sets for binary
classification (see Section~\ref{sec:binary}), and symmetric and
complete hypothesis sets for multi-class classification (see
Section~\ref{sec:multi}).

\begin{table}[t]
  \centering
   \resizebox{\textwidth}{!}{
  \begin{tabular}{@{\hspace{0cm}}lll@{\hspace{0cm}}}
    \toprule
    $\Phi(u)$  & margin-based losses $\ell$   & $\sH$-Consistency bounds\\
    \midrule
    $e^{u}$ & $e^{-yh(x)}$  & $\sE_{\ell_{0-1}}(h) - \sE^*_{\ell_{0-1}}(\sH) + \sM_{\ell_{0-1}}(\sH) \leq \sqrt{2}\paren*{\sE_{\ell}(h)
-\sE^*_{\ell}(\sH) + \sM_{\ell}(\sH)}^{\frac12}$ \\
   $ \log(1 + e^{u})$   & $ \log(1 + e^{-yh(x)})$ & $\sE_{\ell_{0-1}}(h) - \sE^*_{\ell_{0-1}}(\sH) + \sM_{\ell_{0-1}}(\sH) \leq \sqrt{2}\paren*{\sE_{\ell}(h)
-\sE^*_{\ell}(\sH) + \sM_{\ell}(\sH)}^{\frac12}$    \\
  $\max\curl*{0, 1 + u}^2$   & $\max\curl*{0, 1 - yh(x)}^2$ & $\sE_{\ell_{0-1}}(h) - \sE^*_{\ell_{0-1}}(\sH) + \sM_{\ell_{0-1}}(\sH) \leq \paren*{\sE_{\ell}(h)
-\sE^*_{\ell}(\sH) + \sM_{\ell}(\sH)}^{\frac12}$    \\
    $\max\curl*{0, 1 + u}$   & $\max\curl*{0, 1 - yh(x)}$ & $\sE_{\ell_{0-1}}(h) - \sE^*_{\ell_{0-1}}(\sH) + \sM_{\ell_{0-1}}(\sH) \leq \sE_{\ell}(h)
-\sE^*_{\ell}(\sH) + \sM_{\ell}(\sH)$    \\

    \bottomrule
  \end{tabular}
  }
  \vskip 0.1in
  \caption{Examples of $\sH$-consistency bounds for binary margin-based losses.}
\label{tab:bounds-binary}
\end{table}

\begin{table}[t]
  \centering
   \resizebox{\textwidth}{!}{
  \begin{tabular}{@{\hspace{0cm}}lll@{\hspace{0cm}}}
    \toprule
    $\Phi(u)$  & Comp-sum losses $\ell$   & $\sH$-Consistency bounds\\
    \midrule
    $\frac{1 - u}{u}$ & $\sum_{y'\neq y} e^{h(x, y') - h(x, y)}$  & $\sE_{\ell_{0-1}}(h) - \sE^*_{\ell_{0-1}}(\sH) + \sM_{\ell_{0-1}}(\sH) \leq \sqrt{2}\paren*{\sE_{\ell}(h)
-\sE^*_{\ell}(\sH) + \sM_{\ell}(\sH)}^{\frac12}$ \\
    $-\log(u)$   & $-\log\paren*{\frac{e^{h(x, y)}}{\sum_{y'\in \sY}e^{h(x, y')}}}$ & $\sE_{\ell_{0-1}}(h) - \sE^*_{\ell_{0-1}}(\sH) + \sM_{\ell_{0-1}}(\sH) \leq \sqrt{2}\paren*{\sE_{\ell}(h)
-\sE^*_{\ell}(\sH) + \sM_{\ell}(\sH)}^{\frac12}$    \\
    $\frac{1}{\alpha}\bracket*{1 - u^{\alpha}}$    & $\frac{1}{\alpha}\bracket*{1 - \bracket*{\frac{e^{h(x, y)}}
    {\sum_{y'\in  \sY} e^{h(x, y')}}}^{\alpha}}$  &  $\sE_{\ell_{0-1}}(h) - \sE^*_{\ell_{0-1}}(\sH) + \sM_{\ell_{0-1}}(\sH) \leq \sqrt{2n^{\alpha}}\paren*{\sE_{\ell}(h)
-\sE^*_{\ell}(\sH) + \sM_{\ell}(\sH)}^{\frac12}$  \\
    $1 - u$ & $ 1 - \frac{e^{h(x, y)}}{\sum_{y'\in \sY} e^{h(x, y')}}$ & $\sE_{\ell_{0-1}}(h) - \sE^*_{\ell_{0-1}}(\sH) + \sM_{\ell_{0-1}}(\sH) \leq n \paren*{\sE_{\ell}(h)
-\sE^*_{\ell}(\sH) + \sM_{\ell}(\sH)}$    \\
    \bottomrule
  \end{tabular}
  }
  \vskip 0.1in
  \caption{Examples of $\sH$-consistency bounds for comp-sum losses.}
\label{tab:bounds-comp}
\end{table}

\begin{table}[t]
  \centering
   \resizebox{\textwidth}{!}{
  \begin{tabular}{@{\hspace{0cm}}lll@{\hspace{0cm}}}
    \toprule
    $\Phi(u)$  & Constrained losses $\ell$   & $\sH$-Consistency bounds\\
    \midrule
    $e^{u}$ & $\sum_{y'\neq y}e^{h(x, y')}$  & $\sE_{\ell_{0-1}}(h) - \sE^*_{\ell_{0-1}}(\sH) + \sM_{\ell_{0-1}}(\sH) \leq \sqrt{2}\paren*{\sE_{\ell}(h)
-\sE^*_{\ell}(\sH) + \sM_{\ell}(\sH)}^{\frac12}$ \\
   $\max\curl*{0, 1 + u}^2$   & $\sum_{y'\neq y}\max \curl*{0, 1 + h(x, y')}^2$ & $\sE_{\ell_{0-1}}(h) - \sE^*_{\ell_{0-1}}(\sH) + \sM_{\ell_{0-1}}(\sH) \leq \paren*{\sE_{\ell}(h)
-\sE^*_{\ell}(\sH) + \sM_{\ell}(\sH)}^{\frac12}$    \\
  $(1 + u)^2$   & $\sum_{y'\neq y} \paren*{1 + h(x, y')}^2$ & $\sE_{\ell_{0-1}}(h) - \sE^*_{\ell_{0-1}}(\sH) + \sM_{\ell_{0-1}}(\sH) \leq \paren*{\sE_{\ell}(h)
-\sE^*_{\ell}(\sH) + \sM_{\ell}(\sH)}^{\frac12}$    \\
    $\max\curl*{0, 1 + u}$   & $\sum_{y'\neq y}\max \curl*{0, 1 + h(x, y')}$ & $\sE_{\ell_{0-1}}(h) - \sE^*_{\ell_{0-1}}(\sH) + \sM_{\ell_{0-1}}(\sH) \leq \sE_{\ell}(h)
-\sE^*_{\ell}(\sH) + \sM_{\ell}(\sH)$    \\

    \bottomrule
  \end{tabular}
  }
\vskip .1in
  \caption{Examples of $\sH$-consistency bounds for constrained losses with
$\sum_{y\in \sY} h(x, y) = 0$.}
\label{tab:bounds-cstnd}
\end{table}

\section{Comparison with excess error bounds}
\label{app:excess-bounds}

Excess error bounds can be used to derive
bounds for a hypothesis set $\sH$ expressed in terms of the
approximation error. Here, we show, however, that, the resulting
bounds are looser than $\sH$-consistency bounds.

Fix a target loss function $\ell_2$ and a surrogate loss
$\ell_1$. Excess error bounds, also known as \emph{surrogate regret bounds},
are bounds relating the excess errors of these loss functions of the
following form:
\begin{equation}
\label{eq:excess-error-bound}
\forall h \in \sH_{\rm{all}}, \quad
	\psi\paren*{\sE_{\ell_2}(h) - \sE^*_{\ell_2}(\sH_{\rm{all}})}
    	\leq \sE_{\ell_1}(h) - \sE^*_{\ell_1}(\sH_{\rm{all}}),
\end{equation}
where $\psi\colon \Rset_+ \to \Rset_+$ is a non-decreasing and convex
function on $\Rset_+$. Recall that as shown in
\eqref{eq:excess-error-decomp}, the excess error can be written as the
sum of the estimation error and the approximation error. Thus, the
excess error bound can be equivalently expressed as follows:
\begin{equation}
\label{eq:excess-error-bound-equiv}
\forall h \in \sH_{\rm{all}}, \quad
	\psi\paren*{\sE_{\ell_2}(h) - \sE^*_{\ell_2}(\sH) + \sA_{\ell_2}(\sH)}
    	\leq \sE_{\ell_1}(h) - \sE^*_{\ell_1}(\sH) + \sA_{\ell_1}(\sH).
\end{equation}
In Section~\ref{sec:min}, we have shown that the minimizabiliy gap can
be upper bounded by the approximation error $\sM_{\ell}(\sH)\leq
\sA(\sH)$ and is in general a finer quantity for a surrogate loss
$\ell_1$. However, we will show that for a target loss $\ell_2$ that
is \emph{discrete}, the minimizabiliy gap in general coincides with
the approximation error.
\begin{definition}
We say that a target loss $\ell_2$ is \emph{discrete} if we can write
$\ell_2(h, x, y) = \sfL(\hh(x), y)$ for some binary function
$\sfL\colon \sY\times\sY \to \Rset_{+}$.
\end{definition}
In other words, a discrete target loss $ \ell_2 $ is explicitly a
function of both the prediction $ \hh(x) $ and the true label $ y $,
where both belong to the label space $ \sY $. Consequently, it can
assume at most $ n^2 $ distinct discrete values.

Next, we demonstrate that for such discrete target loss functions, if
for any instance, the set of predictions generated by the hypothesis
set completely spans the label space, then the minimizability gap is
precisely equal to the approximation error. For convenience, we denote
by $\sfH(x)$ the set of predictions generated by the hypothesis set on
input $x \in \sX$, defined as $ \sfH(x) = \curl*{\hh(x)\colon h \in
  \sH}$
\begin{theorem}
\label{thm:min-discrete}
Given a discrete target loss function $\ell_2$. Assume that the
hypothesis set $\sH$ satisfies, for any $x \in \sX$, $\sfH(x) =
\sY$. Then, we have $\sI_{\ell_2}(\sH) = 0$ and $\sM_{\ell_2}(\sH) =
\sA_{\ell_2}(\sH)$.
\end{theorem}
\begin{proof}
As shown in Section~\ref{sec:min}, the minimizability gap can be
decomposed in terms of the approximation error and the difference of
pointwise infima:
\begin{align*}
  \sM_{\ell_{2}}(\sH)
  & = \sA_{\ell_{2}}(\sH) - \sI_{\ell_{2}}(\sH)\\
  & = \sA_{\ell_{2}}(\sH) - \E_x \bracket[\Big]{\sC^*_{\ell_{2}}(\sH, x) - \sC^*_{\ell_{2}}(\sH_{\rm{all}}, x)}.
\end{align*}
By definition and the fact that $\ell_2$ is discrete, the conditional error can be written as
\begin{align*}
\sC_{\ell_2}(h,x) = \sum_{y\in \sY} p(x,y) \ell_2(h, x, y) = \sum_{y\in \sY} p(x,y) \sfL(\hh(x), y).
\end{align*}
Thus, for any $x\in \sX$, the best-in-class conditional error can be expressed as 
\begin{equation*}
\sC_{\ell_2}^*(\sH,x) = \inf_{h\in \sH} \sum_{y\in \sY} p(x,y) \sfL(\hh(x), y) = \inf_{y' \in \sfH(x)} \sum_{y\in \sY} p(x,y) \sfL(y', y).
\end{equation*}
By the assumption that $\sfH(x) = \sY$, we obtain
\begin{equation*}
\forall x\in \sX, \quad \sC_{\ell_2}^*(\sH,x) = \inf_{y' \in \sfH(x)} \sum_{y\in \sY} p(x,y) \sfL(y', y) = \inf_{y' \in \sY} \sum_{y\in \sY} p(x,y) \sfL(y', y) = \sC_{\ell_2}^*(\sH_{\rm{all}},x).
\end{equation*}
Therefore, $\sI_{\ell_{2}}(\sH) = \E_x \bracket[\Big]{\sC^*_{\ell_{2}}(\sH, x) - \sC^*_{\ell_{2}}(\sH_{\rm{all}}, x)} = 0$ and $\sM_{\ell_{2}}(\sH) = \sA_{\ell_{2}}(\sH)$.
\end{proof}
By Theorem~\ref{thm:min-discrete}, for a target loss $\ell_2$ that is
discrete and hypothesis sets $\sH$ modulo mild assumptions, the
minimizabiliy gap coincides with the approximation error. In such
cases, by comparing an excess error bound
\eqref{eq:excess-error-bound-equiv} with the $\sH$-consistency bound
\eqref{eq:H-consistency-bound}:
\begin{align*}
&\text{Excess error bound:} \quad 
	\psi\paren*{\sE_{\ell_2}(h) - \sE^*_{\ell_2}(\sH) + \sA_{\ell_2}(\sH)}
    	\leq \sE_{\ell_1}(h) - \sE^*_{\ell_1}(\sH) + \sA_{\ell_1}(\sH)\\
&\text{$\sH$-consistency bound:} \quad
\psi \paren*{\sE_{\ell_2}(h) - \sE^*_{\ell_2}(\sH) + \sM_{\ell_2}(\sH)} \leq \sE_{\ell_1}(h) - \sE^*_{\ell_1}(\sH) + \sM_{\ell_1}(\sH),
\end{align*}
we obtain that the left-hand side of both bounds are equal (since
$\sM_{\ell_2}(\sH) = \sA_{\ell_2}(\sH) $), while the right-hand side
of the $\sH$-consistency bound is always upper bounded by and can be
finer than the right-hand side of the excess error bound (since
$\sM_{\ell_1}(\sH) \leq \sA_{\ell_1}(\sH) $), which implies that
excess error bounds (or surrogate regret bounds) are in general
inferior to $\sH$-consistency bounds.

\section{Polyhedral losses versus smooth losses}
\label{app:poly-smooth}

Since $\sH$-consistency bounds subsume excess error bounds as a
special case (Appendix~\ref{app:excess-bounds}), the linear growth
rate of polyhedral loss excess error bounds
(\citet{finocchiaro2019embedding}) also dictates a linear growth rate
for polyhedral $\sH$-consistency bounds, if they exist. This is
illustrated by the hinge loss or $\rho$-margin loss which have been
shown to benefit from $\sH$-consistency bounds
\citep{awasthi2022h}.

\ignore{
As mentioned in Appendix~\ref{app:excess-bounds}, $\sH$-consistency bounds include excess error bounds as a special case when $\sH = \sH_{\rm{all}}$. Therefore, the linear growth rate of the excess error bound of polyhedral losses shown by \citet{finocchiaro2019embedding} also implies that the growth rate for the $\sH$-consistency bound in the polyhedral case is linear. 
}

Here, we compare in more detail polyhedral losses and the smooth losses. Assume that a hypothesis set $\sH$ is complete and thus $\sfH(x) =
\sY$ for any $x \in \sX$. By Theorem~\ref{thm:min-discrete}, we have $\sA_{\ell_{0-1}}(\sH) = \sM_{\ell_{0-1}}(\sH)$. As shown by \citet[Theorem~3]{frongillo2021surrogate}, a Bayes-consistent polyhedral loss $\Phi_{\rm{poly}}$ admits the following linear excess error bound, for some $\beta_1 > 0$,
\begin{equation}
\label{eq:bound-poly}
\forall h\in \sH,\, \beta_1 \paren*{\sE_{\ell_{0-1}}(h) - \sE^*_{\ell_{0-1}}(\sH) + \sM_{\ell_{0-1}}(\sH)}
    	\leq \sE_{\Phi_{\rm{poly}}}(h) - \sE^*_{\Phi_{\rm{poly}}}(\sH) + \sA_{\Phi_{\rm{poly}}}(\sH).
\end{equation}
However, for a smooth loss $\Phi_{\rm{smooth}}$, if it satisfies the condition of Theorem~\ref{thm:binary-lower},  $\Phi_{\rm{smooth}}$ admits the following $\sH$-consistency bound:
\begin{equation}
\label{eq:bound-smooth}
\forall h\in \sH,\, \sT\paren*{\sE_{\ell_{0-1}}(h) - \sE^*_{\ell_{0-1}}(\sH) + \sM_{\ell_{0-1}}(\sH)}
    	\leq \sE_{\Phi_{\rm{smooth}}}(h) - \sE^*_{\Phi_{\rm{smooth}}}(\sH) + \sM_{\Phi_{\rm{smooth}}}(\sH).
\end{equation}
where $\sT(t) = \Theta(t^2)$. Therefore, our theory offers a principled basis for comparing polyhedral losses \eqref{eq:bound-poly} and smooth losses \eqref{eq:bound-smooth}, which depends on the following factors:
\begin{itemize}
    \item The growth rate: linear for polyhedral losses, while square-root for smooth losses.
    
    \item The optimization property: smooth losses are more favorable for optimization compared to polyhedral losses, in particular with deep neural networks.
    
    \item The approximation theory: the approximation error $ \sA_{\Phi_{\rm{poly}}}(\sH)$ appears on the right-hand side of the bound for polyhedral losses, whereas a finer quantity, the minimizability gap $\sM_{\Phi_{\rm{smooth}}}(\sH)$, is present on the right-hand side of the bound for smooth losses.
\end{itemize}

\section{Comparison of minimizability gaps across comp-sum losses}

\label{app:M-gaps-comp}
For $\ell_{\tau}^{\rm{comp}}$ loss functions, $\tau \in [0, 2)$, we
  can characterize minimizability gaps as follows.

\begin{restatable}
  {theorem}{GapUpperBoundDetermi}
\label{Thm:gap-upper-bound-determi}
Assume that for any $x \in \sX$, we have $\curl*{\paren*{h(x, 1),
    \ldots, h(x, n)}\colon h \in \sH}$ = $[-\Lambda,
  +\Lambda]^n$. Then, for comp-sum losses $\ell_{\tau}^{\rm{comp}}$
and any deterministic distribution, the minimizability gaps can be
expressed as follows:
\begin{align}
\sM_{\ell_{\tau}^{\rm{comp}}}(\sH)
\leq \wt \sM_{\ell_{\tau}^{\rm{comp}}}(\sH) = f_{\tau} \paren*{\sR^*_{\ell_{\tau = 0}^{\rm{comp}}}(\sH)} - f_{\tau} \paren*{\sC^*_{\ell_{\tau = 0}^{\rm{comp}}}(\sH, x)},
\end{align}
where $f_{\tau}(u) = \log(1 + u) 1_{\tau = 1} + \frac{1}{1 - \tau}
\paren*{(1 + u)^{1 - \tau} - 1} 1_{\tau \neq 1}$ and
$\sC^*_{\ell_{\tau = 0}^{\rm{comp}}}(\sH, x) = e^{-2 \Lambda}(n -
1)$. Moreover, $\wt \sM_{\ell_{\tau}^{\rm{comp}}}(\sH)$ is a
non-increasing function of $\tau$.
\end{restatable}

\begin{proof}
Since $f_{\tau}$ is concave and non-decreasing, and the equality $\ell_{\tau} = f_{\tau} \paren*{\ell_{\tau = 0}}$ holds, the minimizability
gaps can be upper bounded as follows, for any $\tau \geq $0,
\begin{align*}
\sM_{\ell_{\tau}^{\rm{comp}}}(\sH)
\leq f_{\tau}\paren*{\sR^*_{\ell_{\tau=0}^{\rm{comp}}}(\sH)} - \E_x[\sC^*_{\ell_{\tau}^{\rm{comp}}}(\sH, x)].
\end{align*}
Since the distribution is deterministic, the conditional error can be expressed as follows:
\begin{equation}
\label{eq:cond-comp-sum-determi}
\begin{aligned}
\sC_{\ell_{\tau}^{\rm{comp}}}(h, x)  =  & f_{\tau}\paren*{\sum_{y'\neq y_{\max}}\exp\paren*{h(x,y')-h(x,y_{\max})}}
\end{aligned}
\end{equation}
where $y_{\max} = \argmax p(x,y)$.
Using the fact that $f_{\tau}$ is increasing for any $\tau>0$, the hypothesis
$
h^*\colon (x, y) \mapsto \Lambda 1_{y = y_{\max}} - \Lambda 1_{y \neq y_{\max}}
$
achieves the best-in-class conditional error.
Thus, 
\begin{align*}
\sC^*_{\ell_{\tau}^{\rm{comp}}}(\sH, x) 
=  \sC_{\ell_{\tau}^{\rm{comp}}}(h^*, x) = f_{\tau}\paren*{\sC^*_{\ell_{\tau = 0}^{\rm{comp}}}(\sH, x)}
\end{align*}
where $\sC^*_{\ell_{\tau = 0}^{\rm{comp}}}(\sH, x) = e^{-2\Lambda}(n - 1)$. Therefore,
\begin{align*}
\sM_{\ell_{\tau}^{\rm{comp}}}(\sH)
\leq f_{\tau}\paren*{\sR^*_{\ell_{\tau = 0}^{\rm{comp}}}(\sH)} - f_{\tau}\paren*{\sC^*_{\ell_{\tau = 0}^{\rm{comp}}}(\sH, x)}.
\end{align*}
This completes the first part of the proof.
Using the fact that $\tau \mapsto f_{\tau} (u_1) - f_{\tau}(u_2)$ is a non-increasing function of $\tau$ for any $u_1 \geq u_2 \geq 0$, the second proof is completed as well.
\end{proof}

The theorem shows that for comp-sum
loss functions $\ell_{\tau}^{\rm{comp}}$, the minimizability gaps are
non-increasing with respect to $\tau$. Note that $\Phi^{\tau}$
satisfies the conditions of Theorem~\ref{thm:comp-lower} for any $\tau
\in [0, 2)$. Therefore, focusing on behavior near zero (ignoring
  constants), the theorem provides a principled comparison of
  minimizability gaps and $\sH$-consistency bounds across different
  comp-sum losses.

\section{Small surrogate minimizability gaps}
\label{app:small-M-gaps}

While minimizability gaps vanish in special scenarios (e.g.,
unrestricted hypothesis sets, best-in-class error matching Bayes
error), we now seek broader conditions for zero or small surrogate
minimizability gaps to make our bounds more meaningful.
  
\subsection{Small surrogate minimizability gaps: binary classification}
\label{app:small-M-gaps-binary}

We first study binary classification, with
multi-class results given in Appendix~\ref{app:small-M-gaps-multi}. We
address pointwise surrogate losses which take the form $\ell(h(x), y)$
for a labeled point $(x, y)$.
We write $A = \curl*{h(x) \colon h \in \sH}$ to denote the set of
predictor values at $x$, which we assume to be independent of
$x$.

\textbf{Deterministic scenario}. We first consider the deterministic
scenario, where the conditional probability $p(x,y)$ is either zero or
one. For a deterministic distribution, we denote by $\sX_+$ the subset
of $\sX$ over which the label is $+1$ and by $\sX_-$ the subset of
$\sX$ over which the label is $-1$. For convenience, let $\ell_+ =
\inf_{\alpha \in A} \ell(\alpha, +1)$ and $\ell_- = \inf_{\alpha \in
  A} \ell(\alpha, -1)$.

\begin{restatable}{theorem}{ZeroMinGap}
\label{th:ZeroMinGap}
Assume that $\sD$ is deterministic and that the best-in-class error is
achieved by some $h^* \in \sH$. Then, the minimizability gap is null,
$\sM(\sH) = 0$, iff
\begin{align*}
  \ell(h^*(x), +1)  = \ell_+ \text{ a.s.\ over $\sX_+$}, \quad
  \ell(h^*(x), -1)  = \ell_- \text{ a.s.\ over $\sX_-$}.
\end{align*}
If further $\alpha \mapsto \ell(\alpha, +1)$ and $\alpha \mapsto
\ell(\alpha, -1)$ are injective and $\ell_+ = \ell(\alpha_+, +1)$,
$\ell_- = \ell(\alpha_-, -1)$, then, the condition is equivalent to
$h^*(x) = \alpha_+ 1_{x \in \sX_+} + \alpha_- 1_{x \in \sX_-}$
Furthermore, the minimizability gap is bounded by $\e$ iff $p
\paren*{\E \bracket*{\ell(h^*(x), +1) \mid y = +1} - \ell_+ } + (1 -
p) \paren*{\E\bracket*{\ell(h^*(x), -1) \mid y = -1} - \ell_-} \leq \e
$. In particular, the condition implies:
\begin{align*}
  & \E \bracket*{\ell(h^*(x), +1) \mid y = +1} - \ell_{+} \leq \frac{\e}{p}
  \quad \text{and} \quad
  \E \bracket*{\ell(h^*(x), -1) \mid y = -1} - \ell_- \leq \frac{\e}{1 - p}.
\end{align*}
\end{restatable}
\begin{proof}
By definition of $h^*$, using the shorthand $p = \P[y = +1]$, we can write
  \begin{align*}  
  \inf_{h \in \sH} \E[\ell(h(x), y)]
  & = \E[\ell(h^*(x), y)]\\
  & = p \E[\ell(h^*(x), +1) \mid y = +1]
  + (1 - p) \E[\ell(h^*(x), -1) \mid y = -1].
\end{align*}
  Since the distribution is deterministic, the expected pointwise infimum
  can be rewritten as follows:
\begin{align*}  
  \E_{x}\bracket*{\inf_{h \in \sH} \E_{y}[\ell(h(x), y) \mid x ]}
  =  \E_{x}\bracket*{\inf_{\alpha \in A} \E_{y}[\ell(\alpha, y) \mid x]}
  & = p \inf_{\alpha \in A} \ell(\alpha, +1)
  + (1 - p) \inf_{\alpha \in A} \ell(\alpha, -1)\\
  & = p \ell_+
  + (1 - p) \ell_-,
\end{align*}
where $\ell_+ = \inf_{\alpha \in A} \ell(\alpha, +1)$
and $\ell_- = \inf_{\alpha \in A} \ell(\alpha, -1)$.
Thus, we have
\begin{align*}
  \sM(\sH)
  & = p \E\bracket*{\ell(h^*(x), +1) - \ell_+ \mid y = +1}
  + (1 - p) \E\bracket*{\ell(h^*(x), -1) - \ell_- \mid y = -1}.
\end{align*}
In view of that, since, by definition of $\ell_+$ and $\ell_-$, the
expressions within the conditional expectations are non-negative, the
equality $\sM(\sH) = 0$ holds iff $\ell(h^*(x), +1) - \ell_+ = 0$
almost surely for any $x$ in $\sX_+$ and $\ell(h^*(x), -1) - \ell_- =
0$ almost surely for any $x$ in $\sX_-$. This completes the first part of the proof. 
Furthermore,
$\sM(\sH) \leq \e$ is equivalent to
\[
p \E\bracket*{\ell(h^*(x), +1) - \ell_+ \mid y = +1}
+ (1 - p) \E\bracket*{\ell(h^*(x), -1) - \ell_- \mid y = -1} \leq \e
\]
that is
\[
p \paren*{\E \bracket*{\ell(h^*(x), +1) \mid y = +1} - \ell_+ }
+ (1 - p) \paren*{\E\bracket*{\ell(h^*(x), -1) \mid y = -1} - \ell_-} \leq \e.
\]
In light of the non-negativity of the expressions, this implies in
particular:
\begin{align*}
  & \E \bracket*{\ell(h^*(x), +1) \mid y = +1} - \ell_+ \leq \frac{\e}{p}
  \quad \text{and} \quad
  \E \bracket*{\ell(h^*(x), -1) \mid y = -1} - \ell_- \leq \frac{\e}{1 - p}.
\end{align*}
This completes the second part of the proof.
\end{proof}

\setlength{\intextsep}{0pt}
\setlength{\columnsep}{10pt}
\begin{wrapfigure}{r}{0.25\textwidth}
%\vskip -.1in
  \includegraphics[width=0.25\textwidth]{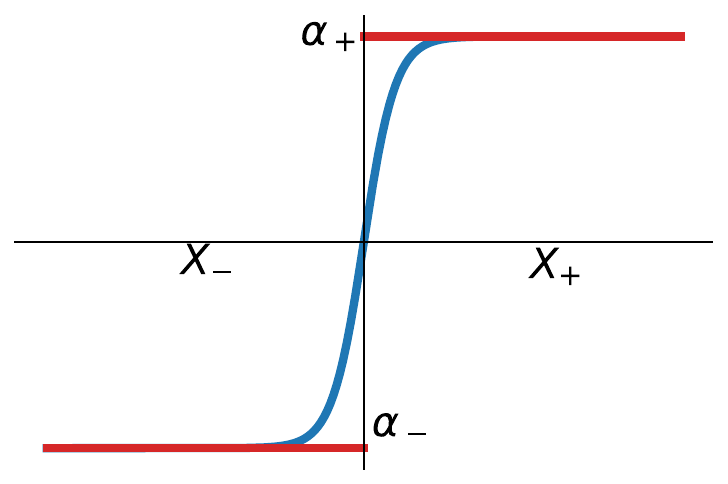}
  \vskip -.05in
\captionsetup{format=plain}
\caption{Approximation provided by sigmoid activation function.}
\label{fig:illustration}
\end{wrapfigure}
The theorem suggests that, under those assumptions, for the surrogate
minimizability gap to be zero, the best-in-class hypothesis must be
piecewise constant with specific values on $\sX_+$ and $\sX_-$. The
existence of such a hypothesis in $\sH$ depends both on the complexity
of the decision surface separating $\sX_+$ and $\sX_-$ and on that of
the hypothesis set $\sH$.  More generally, when the best-in-class
classifier $\e$-approximates $\alpha_+$ over $\sX_+$ and $\alpha_-$
over $\sX_-$, then the minimizability gap is bounded by
$\e$.\ignore{The existence of such a hypothesis in $\sH$ depends on
  the complexity of the decision surface.}  As an example, when the
decision surface is a hyperplane, a hypothesis set of linear functions
combined with a sigmoid activation function can provide such a good
approximation (see Figure~\ref{fig:illustration} for an illustration
in a simple case).

\textbf{Stochastic scenario}. Here, we present a general result that
is a direct extension of that of the deterministic scenario.  We show
that the minimizability gap is zero when there exists $h^*\in \sH$
that matches $\alpha^*(x)$ for all $x$, where $\alpha^*(x)$ is the
minimizer of the conditional error.  We also show that the
minimizability gap is bounded by $\e$ when there exists $h^*\in \sH$
whose conditional error $\e$-approximates best-in-class conditional
error for all $x$.

\begin{restatable}{theorem}{ZeroMinGapStochastic}
\label{th:ZeroMinGapStochastic}
  The best-in-class error is achieved by some $h^*\in \sH$ and the
  minimizability gap is null, $\sM(\sH) = 0$, iff there exists $h^*\in
  \sH$ such that for all $x$,
  \begin{align}
  \label{eq:cond-zero-stochastic}
  \E_{y}[\ell(h^*(x), y) \mid x]
  = \inf_{\alpha \in A} \E_{y}[\ell(\alpha, y) \mid x] \text{ a.s.\ over $\sX$}.
  \end{align}
  If further $\alpha \mapsto \E_{y}[\ell(\alpha, y) \mid x]$ is injective
  and $\inf_{\alpha \in A} \E_{y}[\ell(\alpha, y) \mid x] = \E_{y}[\ell(\alpha^*(x), y) \mid x]$, then, the condition is equivalent to
  $h^*(x) = \alpha^*(x) \text{ a.s.\ for $x \in \sX$}$. Furthermore,
  the minimizability gap is bounded by $\e$, $\sM(\sH) \leq \e$, iff
  there exists $h^*\in \sH$ such that
\begin{align}
\label{eq:cond-epsilon-stochastic}
\E_{x}\bracket*{\E_{y}[\ell(h^*(x), y) \mid x]
  - \inf_{\alpha \in A} \E_{y}[\ell(\alpha, y) \mid x]} \leq \e.
\end{align}
\end{restatable}
\begin{proof}
Assume that the best-in-class
  error is achieved by some $h^*\in \sH$.
Then, we can write
  \begin{align*}  
  \inf_{h \in \sH} \E[\ell(h(x), y)]
   = \E[\ell(h^*(x), y)]=\E_{x}\bracket*{\E_{y}[\ell(h^*(x), y) \mid x]}.
\end{align*}
The expected pointwise infimum
  can be rewritten as follows:
\begin{align*}  
  \E_{x}\bracket*{\inf_{h \in \sH} \E_{y}[\ell(h(x), y) \mid x]}
  =  \E_{x}\bracket*{\inf_{\alpha \in A} \E_{y}[\ell(\alpha, y) \mid x]}.
\end{align*}
Thus, we have
\begin{align*}
  \sM(\sH)
  & = \E_{x}\bracket*{\E_{y}[\ell(h^*(x), y) \mid x]
    - \inf_{\alpha \in A} \E_{y}[\ell(\alpha, y) \mid x]}.
\end{align*}
In view of that, since, by the definition of infimum, the
expressions within the marginal expectations are non-negative, the condition that $\sM(\sH) = 0$ implies that  \begin{align*}
  \E_{y}[\ell(h^*(x), y) \mid x] = \inf_{\alpha \in A} \E_{y}[\ell(\alpha, y) \mid x] \text{ a.s.\ over $\sX$}.
  \end{align*}
On the other hand, if there exists $h^*\in \sH$ such that the
condition \eqref{eq:cond-zero-stochastic} holds, then,
\begin{align*}
 \sM(\sH)
  & =  \inf_{h \in \sH} \E[\ell(h(x), y)] -  \E_{x}\bracket*{\inf_{\alpha \in A} \E_{y}[\ell(\alpha, y) \mid x]}\leq \E_{x}\bracket*{\E_{y}[\ell(h^*(x), y) \mid x]-\inf_{\alpha \in A} \E_{y}[\ell(\alpha, y) \mid x]} = 0.
\end{align*}
Since $\sM(\sH)$ is non-negative, the inequality is achieved. Thus, we
have
\begin{align*}
  \sM(\sH) = 0 \text{ and } \inf_{h \in \sH} \E[\ell(h(x), y)] = \E[\ell(h^*(x), y)].
\end{align*}
If there exists $h^*\in \sH$ such that the
condition~\eqref{eq:cond-epsilon-stochastic} holds, then,
\begin{align*}
 \sM(\sH)
  & =  \inf_{h \in \sH} \E[\ell(h(x), y)] -  \E_{x}\bracket*{\inf_{\alpha \in A} \E_{y}[\ell(\alpha, y) \mid x]}\leq \E_{x}\bracket*{\E_{y}[\ell(h^*(x), y) \mid x]-\inf_{\alpha \in A} \E_{y}[\ell(\alpha, y) \mid x]} = \e.
\end{align*}
On the other hand, since we have
\begin{align*}
  \sM(\sH)
  & = \E_{x}\bracket*{\E_{y}[\ell(h^*(x), y) \mid x] - \inf_{\alpha \in A} \E_{y}[\ell(\alpha, y) \mid x]},
\end{align*}
$\sM(\sH) \leq \e$ implies that
\[
\E_{x}\bracket*{\E_{y}[\ell(h^*(x), y) \mid x] - \inf_{\alpha \in A} \E_{y}[\ell(\alpha, y) \mid x]} \leq \e.
\]
This completes the proof.
\end{proof}

In deterministic settings,
condition~\eqref{eq:cond-epsilon-stochastic} coincides with that of
Theorem~\ref{th:ZeroMinGap}. However, in stochastic scenarios, the
existence of such a hypothesis depends on both decision surface
complexity and the conditional distribution's properties.  For
illustration, see Appendix~\ref{app:examples} where we analyze
exponential, logistic (binary), and multi-class logistic losses.
We thoroughly analyzed minimizability gaps, comparing them across
comp-sum losses, and identifying conditions for zero or small gaps.
These findings help inform surrogate loss selection.
In Appendix~\ref{app:excess-bounds}, we further demonstrate the key
role of minimizability gaps in comparing excess bounds with
$\sH$-consistency bounds.  Importantly, combining $\sH$-consistency
bounds with surrogate loss Rademacher complexity bounds allows us to
derive zero-one loss (estimation) learning bounds for surrogate loss
minimizers (see Appendix~\ref{app:generalization-bound}).

\subsection{Small surrogate minimizability gaps: multi-class classification}
\label{app:small-M-gaps-multi}

We consider the multi-class setting with label space $[n] = \curl*{1,
  2, \ldots, n}$. In this setting, the surrogate loss incurred by a
predictor $h$ at a labeled point $(x, y)$ can be expressed by
$\ell(h(x), y)$, where $h(x) = \bracket*{h(x, 1), \ldots, h(x, n)}$ is
the score vector of the predictor $h$.
We denote by $A$ the set of values in $\Rset^n$ taken by the score
vector of predictors in $\sH$ at $x$, which we assume to be
independent of $x$: $A = \curl*{h(x) \colon h \in \sH}$, for all $x
\in \sX$.

\textbf{Deterministic scenario}.
We first consider the deterministic scenario, where the conditional
probability $p(x,y)$ is either zero or one. For a deterministic distribution, we denote by $\sX_{k}$ the subset of
$\sX$ over which the label is $k$. For convenience, let  $\ell_k = \inf_{\alpha \in A} \ell(\alpha, k)$, for any $k\in [n]$.

\begin{restatable}{theorem}{ZeroMinGapMulti}
\label{th:ZeroMinGapMulti}
  Assume that $\sD$ is deterministic and that the best-in-class error
  is achieved by some $h^* \in \sH$. Then, the minimizability gap is
  null, $\sM(\sH) = 0$, iff
  \begin{align*}
  \forall k \in [n],\, \ell(h^*(x), k) & = \ell_k  \text{ a.s.\ over $\sX_{k}$}.
  \end{align*}
  If further $\alpha \mapsto \ell(\alpha, k)$ is injective and
  $\ell_k = \ell(\alpha_{k}, k)$ for all $k \in [n]$,
  then, the condition is equivalent to
  $
    \forall k \in [n],\, h^*(x) =
    \alpha_{k}  \text{ a.s.\ for $x \in \sX_{k}$}.
  $ Furthermore, the minimizability
  gap is bounded by $\e$, $\sM(\sH) \leq \e$, iff
\[
\sum_{k \in [n]}  p_k \paren*{\E \bracket*{\ell(h^*(x), k) \mid y = k} - \ell_k }\leq \e.
\]
In particular, the condition implies:
\begin{align*}
  \E \bracket*{\ell(h^*(x), k) \mid y = k} - \ell_k \leq \frac{\e}{p_k},\, \forall k\in [n],
\end{align*}
\end{restatable}

\begin{proof}
By definition of $h^*$, using the shorthand $p_{k} = \P[y = k]$ for any $k\in [n]$, we can write
  \begin{align*}  
  \inf_{h \in \sH} \E[\ell(h(x), y)] = \E[\ell(h^*(x), y)]= \sum_{k \in [n]}  p_k \E[\ell(h^*(x), k) \mid y = k].
\end{align*}
  Since the distribution is deterministic, the expected pointwise
  infimum can be rewritten as follows:
\begin{align*}  
  \E_{x}\bracket*{\inf_{h \in \sH} \E_{y}[\ell(h(x), y) \mid x]}
  =  \E_{x}\bracket*{\inf_{\alpha \in A} \E_{y}[\ell(\alpha, y) \mid x]} = \sum_{k \in [n]}  p_k \inf_{\alpha \in A} \ell(\alpha, k) = \sum_{k \in [n]}  p_k \ell_k,
\end{align*}
where $\ell_k = \inf_{\alpha \in A} \ell(\alpha, k)$, for any $k\in [n]$.
Thus, we have
\begin{align*}
  \sM(\sH)
  & = \sum_{k \in [n]}  p_k \E\bracket*{\ell(h^*(x), k) - \ell_k \mid y = k}.
\end{align*}
In view of that, since, by definition of $\ell_k$, the expressions
within the conditional expectations are non-negative, the equality
$\sM(\sH) = 0$ holds iff $\ell(h^*(x), k) - \ell_k = 0$ almost surely
for any $x$ in $\sX_k$, $\forall k \in [n]$. Furthermore,
$\sM(\sH) \leq \e$ is equivalent to
\[
\sum_{k \in [n]}  p_k \E\bracket*{\ell(h^*(x), k) - \ell_k \mid y = k} \leq \e
\]
that is
\[
\sum_{k \in [n]}  p_k \paren*{\E \bracket*{\ell(h^*(x), k) \mid y = k} - \ell_k }\leq \e.
\]
In light of the non-negativity of the expressions, this implies in
particular:
\begin{align*}
 \E \bracket*{\ell(h^*(x), k) \mid y = k} - \ell_k \leq \frac{\e}{p_k}, \, \forall k\in [n].
\end{align*}
This completes the proof.
\end{proof}

The theorem suggests that, under those assumptions, for the surrogate
minimizability gap to be zero, the score vector of best-in-class
hypothesis must be piecewise constant with specific values on
$\sX_k$. The existence of such a hypothesis in $\sH$ depends both on
the complexity of the decision surface separating $\sX_k$ and on that
of the hypothesis set $\sH$. The theorem also suggests that when the score vector of best-in-class classifier
$\e$-approximates $\alpha_k$ over $\sX_k$ for any $k\in [n]$,
then the minimizability gap is bounded by $\e$. The existence of such
a hypothesis in $\sH$ depends on the complexity of the decision
surface.

\textbf{Stochastic scenario}.
In the above, we analyze instances featuring small
minimizability gaps in a deterministic setting. Moving forward, we aim
to extend this analysis to the stochastic scenario. We first provide
two general results, which are the direct extensions of that in the
deterministic scenario. The following result shows that the
minimizability gap is zero when there exists $h^*\in \sH$ that matches
$\alpha^*(x)$ for all $x$, where $\alpha^*(x)$ is the minimizer of the
conditional error.  It also shows that the
minimizability gap is bounded by $\e$ when there exists $h^*\in \sH$
whose conditional error $\e$-approximates best-in-class conditional
error for all $x$.

\begin{restatable}{theorem}{ZeroMinGapStochasticMulti}
\label{th:ZeroMinGapStochasticMulti}
  The best-in-class
  error is achieved by some $h^*\in \sH$ and the minimizability
  gap is null, $\sM(\sH) = 0$, iff there exists $h^*\in \sH$ such that
  \begin{align}
  \label{eq:cond-zero-stochastic-Multi}
  \E_{y}[\ell(h^*(x), y) \mid x] = \inf_{\alpha \in A} \E_{y}[\ell(\alpha, y) \mid x] \text{ a.s.\ over $\sX$}.
  \end{align}
  If further $\alpha \mapsto \E_{y}[\ell(\alpha, y) \mid x]$ is injective
  and $\inf_{\alpha \in A} \E_{y}[\ell(\alpha, y) \mid x] = \E_{y}[\ell(\alpha^*(x), y) \mid x]$, then, the condition is equivalent to
  $h^*(x) =
    \alpha^*(x)  \text{ a.s.\ for $x \in \sX$}$. Furthermore, the minimizability
  gap is bounded by $\e$, $\sM(\sH) \leq \e$, iff there exists $h^*\in \sH$ such that
\begin{align}
\label{eq:cond-epsilon-stochastic-Multi}
\E_{x}\bracket*{\E_{y}[\ell(h^*(x), y) \mid x] - \inf_{\alpha \in A} \E_{y}[\ell(\alpha, y) \mid x]} \leq \e.
\end{align}
\end{restatable}
\begin{proof}
Assume that the best-in-class
  error is achieved by some $h^*\in \sH$.
Then, we can write
  \begin{align*}  
  \inf_{h \in \sH} \E[\ell(h(x), y)]
   = \E[\ell(h^*(x), y)]=\E_{x}\bracket*{\E_{y}[\ell(h^*(x), y) \mid x]}.
\end{align*}
The expected pointwise infimum
  can be rewritten as follows:
\begin{align*}  
  \E_{x}\bracket*{\inf_{h \in \sH} \E_{y}[\ell(h(x), y) \mid x]}
  =  \E_{x}\bracket*{\inf_{\alpha \in A} \E_{y}[\ell(\alpha, y) \mid x]}.
\end{align*}
Thus, we have
\begin{align*}
  \sM(\sH)
  & = \E_{x}\bracket*{\E_{y}[\ell(h^*(x), y) \mid x]
    - \inf_{\alpha \in A} \E_{y}[\ell(\alpha, y) \mid x]}.
\end{align*}
In view of that, since, by the definition of infimum, the
expressions within the marginal expectations are non-negative, the condition that $\sM(\sH) = 0$ implies that  \begin{align*}
  \E_{y}[\ell(h^*(x), y) \mid x] = \inf_{\alpha \in A} \E_{y}[\ell(\alpha, y) \mid x] \text{ a.s.\ over $\sX$}.
  \end{align*}
On the other hand, if there exists $h^*\in \sH$ such that the
condition \eqref{eq:cond-zero-stochastic} holds, then,
\begin{align*}
 \sM(\sH)
  & =  \inf_{h \in \sH} \E[\ell(h(x), y)] -  \E_{x}\bracket*{\inf_{\alpha \in A} \E_{y}[\ell(\alpha, y) \mid x]}\leq \E_{x}\bracket*{\E_{y}[\ell(h^*(x), y) \mid x]-\inf_{\alpha \in A} \E_{y}[\ell(\alpha, y) \mid x]} = 0.
\end{align*}
Since $\sM(\sH)$ is non-negative, the inequality is achieved. Thus, we
have
\begin{align*}
  \sM(\sH) = 0 \text{ and } \inf_{h \in \sH} \E[\ell(h(x), y)] = \E[\ell(h^*(x), y)].
\end{align*}
If there exists $h^*\in \sH$ such that the
condition~\eqref{eq:cond-epsilon-stochastic} holds, then,
\begin{align*}
 \sM(\sH)
  & =  \inf_{h \in \sH} \E[\ell(h(x), y)] -  \E_{x}\bracket*{\inf_{\alpha \in A} \E_{y}[\ell(\alpha, y) \mid x]}\leq \E_{x}\bracket*{\E_{y}[\ell(h^*(x), y) \mid x]-\inf_{\alpha \in A} \E_{y}[\ell(\alpha, y) \mid x]} = \e.
\end{align*}
On the other hand, since we have
\begin{align*}
  \sM(\sH)
  & = \E_{x}\bracket*{\E_{y}[\ell(h^*(x), y) \mid x] - \inf_{\alpha \in A} \E_{y}[\ell(\alpha, y) \mid x]},
\end{align*}
$\sM(\sH) \leq \e$ implies that
\[
\E_{x}\bracket*{\E_{y}[\ell(h^*(x), y) \mid x] - \inf_{\alpha \in A} \E_{y}[\ell(\alpha, y) \mid x]} \leq \e.
\]
This completes the proof.
\end{proof}

\subsection{Examples}
\label{app:examples}

Note that when the distribution is assumed to be deterministic, the
condition~\eqref{eq:cond-epsilon-stochastic} and condition~\eqref{eq:cond-epsilon-stochastic-Multi} are reduced to the
condition of Theorem~\ref{th:ZeroMinGap} in binary classification
and that of Theorem~\ref{th:ZeroMinGapMulti} in multi-class
classification, respectively.  In the stochastic scenario, the
existence of such a hypothesis not only depends on the complexity of
the decision surface, but also depends on the distributional
assumption on the conditional distribution $p(x) = \paren*{p(x,
  y)}_{y\in \sY}$, where $p(x, y) = \sD(Y = y \!\mid\! X = x)$ is the
conditional probability of $Y = y$ given $X = x$. In the binary
classification, we have $p(x) = \paren*{p(x, \plus 1),p(x, \minus
  1)}$, where $p(x, \plus 1) + p(x,\minus 1) = 1$. For simplicity, we
use the notation $\eta(x)$ and $1 - \eta(x)$ to represent $p(x, \plus
1)$ and $p(x, \minus 1)$ respectively. In the multi-class
classification with $\sY = \curl*{1, \ldots, n}$, we have $p(x) =
\paren*{p(x, 1), p(x,2), \ldots, p(x,n)}$ where $n$ is the number of
classes. As examples, here too, we examine exponential loss and
logistic loss in binary classification and multi-class logistic loss
in multi-class classification.

\textbf{A. Example: binary classification.} 
Let $\e\in [0,\frac12]$.  We denote by $\sX_+$ the subset of $\sX$
over which $\eta(x)=1-\e$ and by $\sX_-$ the subset of $\sX$ over
which $\eta(x) = \e$. Let $\sH$ be a family of functions $h$ with
$\abs*{h(x)} \leq \Lambda$ for all $x \in \sX$ and such that all
values in $[-\Lambda, +\Lambda]$ can be reached. Thus, $A =
\bracket*{-\Lambda, \Lambda}$ for any $x\in \sX$. Consider the
exponential loss: $\ell(h, x, y) = e^{-yh(x)}$. Then, for any $x\in
\sX$ and $\alpha\in A$, we have
\begin{align*}
\E_{y}[\ell(\alpha, y) \mid x] = \begin{cases}
(1-\e)e^{-\alpha} + \e e^{\alpha} & x\in \sX_{+}\\
\e e^{-\alpha} + (1-\e) e^{\alpha} & x\in \sX_{-}.
\end{cases}
\end{align*}
Thus, it is not hard to see that for any $\e \leq \frac{1}{e^{2\Lambda}+1}$, the infimum $\inf_{\alpha \in A} \E_{y}[\ell(\alpha, y) \mid x]$ can be achieved by $\alpha^*(x) = \begin{cases}
\Lambda & x\in \sX_{+}\\
-\Lambda & x\in \sX_{-}
\end{cases}\in A$. Similarly, for the logistic loss $\ell(h, x, y) = \log\paren*{1 + e^{-yh(x)}}$, we have that
\begin{align*}
\E_{y}[\ell(\alpha, y) \mid x] = \begin{cases}
(1 - \e)\log\paren*{1 + e^{-\alpha}} + \e \log\paren*{1 + e^{\alpha}} & x\in \sX_{+}\\
\e \log\paren*{1 + e^{-\alpha}} + (1 - \e) \log\paren*{1 + e^{\alpha}} & x\in \sX_{-}
\end{cases}
\end{align*}
and for $\e \leq \frac{1}{e^{\Lambda}+1}$, the infimum $\inf_{\alpha \in A} \E_{y}[\ell(\alpha, y) \mid x]$ can be achieved by $\alpha^*(x) = \begin{cases}
\Lambda & x\in \sX_{+}\\
-\Lambda & x\in \sX_{-}
\end{cases}$. Therefore, by Theorem~\ref{th:ZeroMinGapStochastic},
for these distributions and loss functions, when the best-in-class
classifier $h^*$ $\e$-approximates $\alpha_{+} = \Lambda$ over $\sX_+$
and $\alpha_{-} = -\Lambda$ over $\sX_-$, then the minimizability gap
is bounded by $\e$. The existence of such a hypothesis in $\sH$
depends on the complexity of the decision surface. For example, as
previously noted, when the decision surface is characterized by a
hyperplane, a hypothesis set of linear functions, coupled with a
sigmoid activation function, can offer a highly effective
approximation (see Figure~\ref{fig:illustration} for illustration).

\textbf{B. Example: multi-class classification.} 
Let $\e\in [0,\frac12]$. We denote by $\sX_{k}$ the subset of $\sX$
over which $p(x,k) = 1-\e$ and $p(x,j) = \frac{\e}{n-1}$ for $j\neq
k$. Let $\sH$ be a family of functions $h$ with $\abs*{h(x, \cdot)}
\leq \Lambda$ for all $x \in \sX$ and such that all values in
$[-\Lambda, +\Lambda]$ can be reached. Thus, $A = \bracket*{-\Lambda,
  \Lambda}^n$ for any $x\in \sX$. Consider the multi-class logistic
loss: $\ell(h, x, y) = - \log \bracket*{\frac{e^{h(x,y)}}{\sum_{y' \in
      \sY} e^{h(x,y')}}}$. For any $\alpha = \bracket*{\alpha^1,
  \ldots, \alpha^n} \in A$, we denote by $S_k =
\frac{e^{\alpha^k}}{\sum_{k' \in [n]} e^{\alpha^{k'}}}$.  Then, for
any $x\in \sX$ and $\alpha\in A$,
\begin{align*}
\E_{y}[\ell(\alpha, y) \mid x] = - (1 - \e) \log \paren*{S_{k}} -
\frac{\e}{n - 1} \sum_{k'\neq k} \log \paren*{S_{k'}} \text{ if } x\in
\sX_k.
\end{align*}
Thus, it is not hard to see that for any $\e \leq \frac{n -
  1}{e^{2\lambda} + n - 1}$, the infimum $\inf_{\alpha \in A} \E_{y}[\ell(\alpha, y) \mid x]$ can be achieved by $\alpha^*(x) =
\bracket*{-\Lambda, \ldots, \Lambda, \ldots, -\Lambda}$, where
$\Lambda$ occupies the $k$-th position for $x\in \sX_{k}$.  Therefore,
by Theorem~\ref{th:ZeroMinGapStochastic}, for these distributions
and loss functions, when the best-in-class classifier $h^*$
$\e$-approximates $\alpha_{k} = \bracket*{-\Lambda, \ldots,
  \underset{k-\text{th}}{\Lambda}, \ldots, -\Lambda}$ over $\sX_{k}$,
then the minimizability gap is bounded by $\e$. The existence of such
a hypothesis in $\sH$ depends on the complexity of the decision
surface.

\section{Proof for binary margin-based losses (Theorem~\ref{thm:binary-lower})}
\label{app:binary-lower}

\BinaryLower*
\begin{proof}
Since $\Phi$ is convex and in $C^2$, $f_t$ is also convex and differentiable with respect to $u$. For any $t \in [0, 1]$, differentiate $f_t$ with respect to $u$, we have
\begin{equation*}
f'_t(u) = \frac{1 - t}{2} \Phi'(u) - \frac{1 + t}{2} \Phi'(-u).
\end{equation*}
Consider the function $F$ defined over $\Rset^2$ by $F(t, a) = \frac{1 - t}{2} \Phi'(a) - \frac{1+ t}{2} \Phi'(-a)$. Observe that $F(0, 0) = 0$ and that the partial derivative of $F$ with respect to $a$ at $(0, 0)$ is $\Phi''(0) > 0$:
\begin{equation*}
\frac{\partial F}{\partial a}(t, a) = \frac{1 - t}{2} \Phi''(a) + \frac{1 + t}{2} \Phi''(-a), \quad \frac{\partial F}{\partial a}(0, 0) = \Phi''(0) > 0.
\end{equation*}
Consequently, by the implicit function theorem, there exists a continuously differentiable function $\ov a$ such that $F (t, \ov a(t)) = 0$ in a neighborhood $[-\e, \e]$ around zero. Thus, by the convexity of $f_t$ and the definition of $F$, for $t \in [0, \e]$, $\inf_{u \in \Rset} f_t(u)$ is reached by $\ov a(t)$ and we can denote it by $a^*_t$. Then, $a^*_t$ is continuously differentiable over $[0, \epsilon]$. The minimizer $a^*_t$ satisfies the following equality:
\begin{equation}
\label{eq:a_t}
f'_t(a^*_t) = \frac{1 - t}{2} \Phi'(a^*_t) - \frac{1 + t}{2} \Phi'(-a^*_t) = 0.
\end{equation}
Specifically, at $t = 0$, we have $\Phi'(a^*_0) = \Phi'(-a^*_0)$.
Since $\Phi$ is convex, its derivative $\Phi'$ is non-decreasing.
Therefore, if $a^*_0$ were non-zero, then $\Phi'$ would be constant
over the segment $[-\abs*{a^*_0}, \abs*{a^*_0}]$. This would contradict the
condition $\Phi''(0) > 0$, as a constant function cannot have a
positive second derivative at any point. Thus, we must have $a^*_0 =
0$ and since $\Phi'$ is non-decreasing and $\Phi''(0) > 0$, we have $a^*_t > 0$ for all $t \in (0, \e]$. By Theorem~\ref{thm:binary-char} and Taylor's theorem with an integral remainder, 
$\sT$ can be expressed as follows: for any $t \in [0, \e]$,
\begin{align}
\sT(t) 
& = f_t(0) -\inf_{u \in \Rset} f_t(u) \nonumber\\
& = f_t(0) - f_t(a^*_t) \nonumber\\
& = f'_t(a^*_t) (0 - a^*_t) + \int_{a^*_t}^0 (0 - u) f_t''(u) \, du
\tag{$f'_t(a^*_t) = 0$} \nonumber\\
%& = \int_{a^*_t}^0 (0 - u) f''_t(u) \, du \nonumber\\
& = \int_0^{a^*_t} u f''_t(u) \, du \nonumber\\
& = \int_0^{a^*_t} u \bracket*{\frac{1 - t}{2} \Phi''(u) + \frac{1 + t}{2} \Phi''(-u)} \, du.
\label{eq:T}
\end{align}
Since $a^*_t$ is a function of class $C^1$, we can differentiate \eqref{eq:a_t} with respect to $t$, which gives the following equality for any $t$ in $(0, \epsilon]$:
\begin{equation*}
-\frac12 \Phi'(a_t^*) + \frac{1 - t}{2 }\Phi''(a_t^*) \frac{d a_t^*}{d t}(t) - \frac12 \Phi'(-a_t^*) + \frac{1 + t}{2} \Phi''(-a_t^*)
\frac{d a_t^*}{d t}(t) = 0.
\end{equation*}
Taking the limit $t \to 0$ yields
\begin{equation*}
-\frac12 \Phi'(0) + \frac12 \Phi''(0) \frac{d a_t^*}{d t}(0) - \frac12 \Phi'(0) + \frac12 \Phi''(0) \frac{d a_t^*}{d t}(0) = 0.
\end{equation*}
This implies that
\begin{equation*}
\frac{d a_t^*}{d t}(0) = \frac{\Phi'(0)}{\Phi''(0)} > 0.
\end{equation*}
Since $\lim_{t \to 0} \frac{a_t^*}{t} = \frac{d a_t^*}{d t}(0) = \frac{\Phi'(0)}{\Phi''(0)} > 0$, we have $a^*_t = \Theta(t)$. 

Since $\Phi''(0) > 0$ and $\Phi''$ is continuous, there is a non-empty interval $[\minus \alpha, \plus \alpha]$ over which $\Phi''$ is positive. Since $a^*_0 = 0$ and $a^*_t$ is continuous, there exists a sub-interval $[0, \epsilon'] \subseteq [0, \epsilon]$ over which $a^*_t \leq \alpha$. Since $\Phi''$ is continuous, it admits a minimum and a maximum over any compact set and we can define $c = \min_{u \in [-\alpha, \alpha]} \Phi''(u)$ and $C = \max_{u \in [-\alpha, \alpha]} \Phi''(u)$. $c$ and $C$ are both positive since we have $\Phi''(0) > 0$. Thus, for $t$ in $[0, \epsilon']$, by \eqref{eq:T}, the following inequality holds:
\begin{align*}
C \frac{(a^*_t)^2}{2} = \int_0^{a^*_t}  u C \, du \geq \sT(t) &= \int_0^{a^*_t} u \bracket*{\frac{1 - t}{2} \Phi''(u) + \frac{1 + t}{2} \Phi''(-u)} \, du
\geq \int_0^{a^*_t}  u C \, du
= c \frac{(a^*_t)^2}{2}.
\end{align*}
This implies that $\sT(t) = \Theta(t^2)$.
\end{proof}

\section{Proof for comp-sum losses (Theorem~\ref{thm:comp-lower})}
\label{app:comp-lower}
\CompLower*

\begin{proof}
For any $\tau \in \bracket*{\frac1n, \frac12}$, define the function $\sT_\tau$ by
\begin{align*}
  \forall t \in [0, 1], \quad  \sT_\tau(t)
  & = \sup_{|u| \leq \tau} \curl*{\Phi(\tau) - \frac{1 - t}{2} \Phi(\tau + u) - \frac{1 + t}{2}\Phi(\tau - u)}\\
  & = f_{t, \tau}(0) - \inf_{|u| \leq \tau} f_{t, \tau}(u),
\end{align*}
where
\[
f_{t, \tau}(u)
= \frac{1 - t}{2} \Phi_{\tau}(u) + \frac{1 + t}{2} \Phi_{\tau}(-u)
\quad \text{and} \quad
\Phi_{\tau}(u) = \Phi(\tau + u).
\]
We aim to establish a lower bound for $\inf_{\tau \in \bracket*{\frac1n, \frac12}}
\sT_\tau(t)$.  For any fixed $\tau \in \bracket*{\frac1n, \frac12}$, this situation is
parallel to that of binary classification (Theorem~\ref{thm:binary-char}
and Theorem~\ref{thm:binary-lower}), since we have $\Phi'_{\tau}(0) =
\Phi'(\tau) < 0$ and $\Phi''_{\tau}(0) = \Phi''(\tau) > 0$.  Let $a^*_{t, \tau}$ denotes the minimizer of $f_{t, \tau}$ over $\Rset$. By applying Theorem~\ref{thm:a_implicit} to the function $F\colon (t, u, \tau) \mapsto f'_{t, \tau}(u) = \frac{1 - t}{2} \Phi'_{\tau}(u) - \frac{1 + t}{2} \Phi'_{\tau}(-u)$ and the convexity of $f_{t, \tau}$ with respect to $u$, $a^*_{t, \tau}$ exists, is unique and is
continuously differentiable over $[0, t'_0] \times \bracket*{\frac1n, \frac12}$, for some $t'_0 > 0$.
Moreover, by using the fact that $f'_{0, \tau}(\tau) > 0$ and $f'_{0, \tau}(-\tau) < 0$, and the convexity of $f_{0, \tau}$ with respect to $u$, we have $\abs*{a^*_{0, \tau}} \leq \tau$, $\forall \tau \in \bracket*{\frac1n, \frac12}$.  By the continuity of $a^*_{t, \tau}$, we have $\abs*{a^*_{t, \tau}} \leq \tau$ over $[0, t_0] \times \bracket*{\frac1n, \frac12}$, for some $t_0 > 0$ and $t_0 \leq t'_0$. 

Next, we will
leverage the proof of Theorem~\ref{thm:binary-lower}. Adopting a similar
notation, while incorporating the $\tau$ subscript to distinguish
different functions $\Phi_\tau$ and $f_{t, \tau}$, we can write
\[
\forall t \in [0, t_0], \quad
\sT_\tau(t) = \int_0^{-a^*_{t, \tau}} u \bracket*{\frac{1 - t}{2} \Phi''_{\tau}(-u) + \frac{1 + t}{2} \Phi''_{\tau}(u)} \, du.
\]
where $a^*_{t, \tau}$  verifies
\begin{equation}
\label{eq:DerivativeAtZero-comp}
a_{0, \tau}^* = 0 \quad \text{and} \quad
\frac{\partial a_{t, \tau}^*}{\partial t}(0) = \frac{\Phi'_{\tau}(0)}{\Phi''_{\tau}(0)} = c_\tau < 0.
\end{equation} 
We first show the lower bound $\inf_{\tau \in \bracket*{\frac1n, \frac12}} -a_{t,
  \tau}^* = \Omega(t)$.
Given the equalities \eqref{eq:DerivativeAtZero-comp}, it follows that for
any $\tau$, the following holds: $\lim_{t \to 0} \paren*{-a_{t, \tau}^*
  + c_\tau t} = 0$. For any $\tau \in \bracket*{\frac1n, \frac12}$, $t \mapsto
\paren*{-a_{t, \tau}^* + c_\tau t}$ is a continuous function over $[0,
  t_0]$ since $a_{t, \tau}^*$ is a function of class $C^1$. Since the
infimum over a fixed compact set of a family of continuous functions is
continuous, $t \mapsto \inf_{\tau \in \bracket*{\frac1n, \frac12}}\curl*{-a_{t, \tau}^* +
  c_\tau t}$ is continuous.  Thus, for any $\e > 0$, there exists $t_1
> 0$, $t_1 \leq t_0$, such that for any $t \in [0, t_1]$,
\[
\abs*{\inf_{\tau \in \bracket*{\frac1n, \frac12}} \curl*{-a_{t, \tau}^* + c_\tau t}} \leq \e,
\]
which implies
\[
\forall \tau \in \bracket*{\frac1n, \frac12}, \quad
-a_{t, \tau}^*
\geq -c_\tau t - \e
\geq c t - \e,
\]
where $c = \inf_{\tau \in \bracket*{\frac1n, \frac12}} -c_\tau$. Since $\Phi'_{\tau}(0)$ and
$\Phi''_{\tau}(0)$ are positive and continuous functions of $\tau$,
this infimum is attained over the compact set $\bracket*{\frac1n, \frac12}$, leading to $c
> 0$. Since the lower bound holds uniformly over $\tau$, this shows
that for $t \in [0, t_1]$, we have $\inf_{\tau \in \bracket*{\frac1n, \frac12}} -a_{t,
  \tau}^* = \Omega(t)$.

Now, since for any $\tau \in \bracket*{\frac1n, \frac12}$, $-a_{t, \tau}^*$ is a function of
class $C^1$ and thus continuous, its supremum over a compact set,
$\sup_{\tau \in \bracket*{\frac1n, \frac12}} -a_{t, \tau}^*$, is also continuous and is
bounded over $[0, t_1]$ by some $a > 0$. For $|u| \leq a$ and $\tau
\in \bracket*{\frac1n, \frac12}$, we have $\frac12 - a \leq \tau + u \leq \frac12 + a$ and $\frac12 - a \leq
\tau - u \leq \frac12 + a$. Since $\Phi''$ is positive and continuous, it
reaches its minimum $C > 0$ over the compact set $\bracket*{\frac12 - a, \frac12 + a}$.
Thus, we can write 
\begin{align*}
\forall t \in [0, t_1], \forall \tau \in \bracket*{\frac1n, \frac12}, \quad
   \sT_\tau(t)
  & = \int_0^{-a^*_{t, \tau}} u \bracket*{\frac{1 - t}{2} \Phi''_{\tau}(-u) + \frac{1 + t}{2} \Phi''_{\tau}(u)} \, du\\
  & \geq \int_0^{-a^*_{t, \tau}} u \bracket*{\frac{1 - t}{2} C + \frac{1 + t}{2} C} \, du\\
  & = \int_0^{-a^*_{t, \tau}} C u  \, du
  = C \frac{(-a^{*}_{t, \tau})^2}{2}.
\end{align*}
Thus, for $t \leq t_1$, we have
\[
\inf_{\tau \in \bracket*{\frac1n, \frac12}} \sT_\tau(t) \geq C
\frac{(\inf_{\tau \in \bracket*{\frac1n, \frac12}} -a^{*}_{t, \tau})^2}{2} \geq \Omega(t^2).
\]
Similarly, we aim to establish an upper bound for $\inf_{\tau \in \bracket*{\frac1n, \frac12}}
\sT_\tau(t)$.  We first show the upper bound $\sup_{\tau \in \bracket*{\frac1n, \frac12}} -a_{t,
  \tau}^* = O(t)$.
Given the equalities \eqref{eq:DerivativeAtZero-comp}, it follows that for
any $\tau$, the following holds: $\lim_{t \to 0} \paren*{-a_{t, \tau}^*
  + c_\tau t} = 0$. For any $\tau \in \bracket*{\frac1n, \frac12}$, $t \mapsto
\paren*{-a_{t, \tau}^* + c_\tau t}$ is a continuous function over $[0,
  t_0]$ since $a_{t, \tau}^*$ is a function of class $C^1$. Since the
supremum over a fixed compact set of a family of continuous functions is
continuous, $t \mapsto \sup_{\tau \in \bracket*{\frac1n, \frac12}}\curl*{-a_{t, \tau}^* +
  c_\tau t}$ is continuous.  Thus, for any $\e > 0$, there exists $t_1
> 0$, $t_1 \leq t_0$, such that for any $t \in [0, t_1]$,
\[
\abs[\Big]{\sup_{\tau \in \bracket*{\frac1n, \frac12}} \curl*{-a_{t, \tau}^* + c_\tau t}} \leq \e,
\]
which implies
\[
\forall \tau \in \bracket*{\frac1n, \frac12}, \quad
-a_{t, \tau}^*
\leq -c_\tau t + \e
\leq c t + \e,
\]
where $c = \sup_{\tau \in \bracket*{\frac1n, \frac12}} -c_\tau$. Since $\Phi'_{\tau}(0)$ and
$\Phi''_{\tau}(0)$ are positive and continuous functions of $\tau$,
this supremum is attained over the compact set $\bracket*{\frac1n, \frac12}$, leading to $c
> 0$. Since the upper bound holds uniformly over $\tau$, this shows
that for $t \in [0, t_1]$, we have $\sup_{\tau \in \bracket*{\frac1n, \frac12}} -a_{t,
  \tau}^* = O(t)$.

Now, since for any $\tau \in \bracket*{\frac1n, \frac12}$, $-a_{t, \tau}^*$ is a function of
class $C^1$ and thus continuous, its supremum over a compact set,
$\sup_{\tau \in \bracket*{\frac1n, \frac12}} -a_{t, \tau}^*$, is also continuous and is
bounded over $[0, t_1]$ by some $a > 0$. For $|u| \leq a$ and $\tau
\in \bracket*{\frac1n, \frac12}$, we have $\frac12 - a \leq \tau + u \leq \frac12 + a$ and $\frac12 - a \leq
\tau - u \leq \frac12 + a$. Since $\Phi''$ is positive and continuous, it
reaches its maximum $C > 0$ over the compact set $\bracket*{\frac12 - a, \frac12 + a}$.
Thus, we can write
\begin{align*}
  \forall t \in [0, t_1], \forall \tau \in \bracket*{\frac1n, \frac12}, \quad
  \sT_\tau(t)
  & = \int_0^{-a^*_{t, \tau}} u \bracket*{\frac{1 - t}{2} \Phi''_{\tau}(-u) + \frac{1 + t}{2} \Phi''_{\tau}(u)} \, du\\
  & \leq \int_0^{-a^*_{t, \tau}} u \bracket*{\frac{1 - t}{2} C + \frac{1 + t}{2} C} \, du\\
  & = \int_0^{-a^*_{t, \tau}} C u  \, du
  = C \frac{(-a^{*}_{t, \tau})^2}{2}.
\end{align*}
Thus, for $t \leq t_1$, we have
\[
\inf_{\tau \in \bracket*{\frac1n, \frac12}} \sT_\tau(t) \leq C
\frac{(\sup_{\tau \in \bracket*{\frac1n, \frac12}} -a^{*}_{t, \tau})^2}{2} \leq O(t^2).
\]
This completes the proof.
\end{proof}

\section{Proof for constrained losses (Theorem~\ref{thm:cstnd-lower})}
\label{app:cstnd-lower}
\CstndLower*
\begin{proof}
For any $\tau \in [0, A]$, define the function $\sT_\tau$ by
\begin{align*}
  \forall t \in [0, 1], \quad  \sT_\tau(t)
  & = \sup_{u \in \Rset} \curl*{\Phi(\tau) - \frac{1 - t}{2} \Phi(\tau + u) - \frac{1 + t}{2}\Phi(\tau - u)}\\
  & = f_{t, \tau}(0) - \inf_{u \in \Rset} f_{t, \tau}(u),
\end{align*}
where
\[
f_{t, \tau}(u)
= \frac{1 - t}{2} \Phi_{\tau}(u) + \frac{1 + t}{2} \Phi_{\tau}(-u)
\quad \text{and} \quad
\Phi_{\tau}(u) = \Phi(\tau + u).
\]
We aim to establish a lower bound for $\inf_{\tau \in [0, A]}
\sT_\tau(t)$.  For any fixed $\tau \in [0, A]$, this situation is
parallel to that of binary classification (Theorem~\ref{thm:binary-char}
and Theorem~\ref{thm:binary-lower}), since we also have $\Phi'_{\tau}(0) =
\Phi'(\tau) > 0$ and $\Phi''_{\tau}(0) = \Phi''(\tau) > 0$.  Let $a^*_{t, \tau}$ denotes the minimizer of $f_{t, \tau}$ over $\Rset$. By applying Theorem~\ref{thm:a_implicit} to the function $F\colon (t, u, \tau) \mapsto f'_{t, \tau}(u) = \frac{1 - t}{2} \Phi'_{\tau}(u) - \frac{1 + t}{2} \Phi'_{\tau}(-u)$ and the convexity of $f_{t, \tau}$ with respect to $u$, $a^*_{t, \tau}$ exists, is unique and is
continuously differentiable over $[0, t_0] \times [0, A]$, for some $t_0 > 0$. 

Next, we will
leverage the proof of Theorem~\ref{thm:binary-lower}. Adopting a similar
notation, while incorporating the $\tau$ subscript to distinguish
different functions $\Phi_\tau$ and $f_{t, \tau}$, we can write
\[
\forall t \in [0, t_0], \quad
\sT_\tau(t) = \int_0^{a^*_{t, \tau}} u \bracket*{\frac{1 - t}{2} \Phi''_{\tau}(u) + \frac{1 + t}{2} \Phi''_{\tau}(-u)} \, du.
\]
where $a^*_{t, \tau}$  verifies
\begin{equation}
\label{eq:DerivativeAtZero}
a_{0, \tau}^* = 0 \quad \text{and} \quad
\frac{\partial a_{t, \tau}^*}{\partial t}(0) = \frac{\Phi'_{\tau}(0)}{\Phi''_{\tau}(0)} = c_\tau > 0.
\end{equation} 
We first show the lower bound $\inf_{\tau \in [0, A]} a_{t,
  \tau}^* = \Omega(t)$.
Given the equalities \eqref{eq:DerivativeAtZero}, it follows that for
any $\tau$, the following holds: $\lim_{t \to 0} \paren*{a_{t, \tau}^*
  - c_\tau t} = 0$. For any $\tau \in [0, A]$, $t \mapsto
\paren*{a_{t, \tau}^* - c_\tau t}$ is a continuous function over $[0,
  t_0]$ since $a_{t, \tau}^*$ is a function of class $C^1$. Since the
infimum over a fixed compact set of a family of continuous functions is
continuous, $t \mapsto \inf_{\tau \in [0, A]}\curl*{a_{t, \tau}^* -
  c_\tau t}$ is continuous.  Thus, for any $\e > 0$, there exists $t_1
> 0$, $t_1 \leq t_0$, such that for any $t \in [0, t_1]$,
\[
\abs[\Big]{\inf_{\tau \in [0, A]} \curl*{a_{t, \tau}^* - c_\tau t}} \leq \e,
\]
which implies
\[
\forall \tau \in [0, A], \quad
a_{t, \tau}^*
\geq c_\tau t - \e
\geq c t - \e,
\]
where $c = \inf_{\tau \in [0, A]} c_\tau$. Since $\Phi'_{\tau}(0)$ and
$\Phi''_{\tau}(0)$ are positive and continuous functions of $\tau$,
this infimum is attained over the compact set $[0, A]$, leading to $c
> 0$. Since the lower bound holds uniformly over $\tau$, this shows
that for $t \in [0, t_1]$, we have $\inf_{\tau \in [0, A]} a_{t,
  \tau}^* = \Omega(t)$.

Now, since for any $\tau \in [0, A]$, $a_{t, \tau}^*$ is a function of
class $C^1$ and thus continuous, its supremum over a compact set,
$\sup_{\tau \in [0, A]} a_{t, \tau}^*$, is also continuous and is
bounded over $[0, t_1]$ by some $a > 0$. For $|u| \leq a$ and $\tau
\in [0, A]$, we have $A - a \leq \tau + u \leq A + a$ and $A - a \leq
\tau - u \leq A + a$. Since $\Phi''$ is positive and continuous, it
reaches its minimum $C > 0$ over the compact set $[A - a, A + a]$.
Thus, we can write
\begin{align*}
  \forall t \in [0, t_1], \forall \tau \in [0, A], \quad
  \sT_\tau(t)
  & = \int_0^{a^*_{t, \tau}} u \bracket*{\frac{1 - t}{2} \Phi''_{\tau}(u) + \frac{1 + t}{2} \Phi''_{\tau}(-u)} \, du\\
  & \geq \int_0^{a^*_{t, \tau}} u \bracket*{\frac{1 - t}{2} C + \frac{1 + t}{2} C} \, du\\
  & = \int_0^{a^*_{t, \tau}} C u  \, du
  = C \frac{(a^{*}_{t, \tau})^2}{2}.
\end{align*}
Thus, for $t \leq t_1$, we have
\[
\inf_{\tau \in [0, A]} \sT_\tau(t) \geq C
\frac{(\inf_{\tau \in [0, A]} a^{*}_{t, \tau})^2}{2} \geq \Omega(t^2).
\]
Similarly, we aim to establish an upper bound for $\inf_{\tau \in [0, A]}
\sT_\tau(t)$. We first show the upper bound $\sup_{\tau \in [0, A]} a_{t,
  \tau}^* = O(t)$.
Given the equalities \eqref{eq:DerivativeAtZero}, it follows that for
any $\tau$, the following holds: $\lim_{t \to 0} \paren*{a_{t, \tau}^*
  - c_\tau t} = 0$. For any $\tau \in [0, A]$, $t \mapsto
\paren*{a_{t, \tau}^* - c_\tau t}$ is a continuous function over $[0,
  t_0]$ since $a_{t, \tau}^*$ is a function of class $C^1$. Since the
supremum over a fixed compact set of a family of continuous functions is
continuous, $t \mapsto \sup_{\tau \in [0, A]}\curl*{a_{t, \tau}^* -
  c_\tau t}$ is continuous.  Thus, for any $\e > 0$, there exists $t_1
> 0$, $t_1 \leq t_0$, such that for any $t \in [0, t_1]$,
\[
\abs[\Big]{\sup_{\tau \in [0, A]} \curl*{a_{t, \tau}^* - c_\tau t}} \leq \e,
\]
which implies
\[
\forall \tau \in [0, A], \quad
a_{t, \tau}^*
\leq c_\tau t + \e
\leq c t + \e,
\]
where $c = \sup_{\tau \in [0, A]} c_\tau$. Since $\Phi'_{\tau}(0)$ and
$\Phi''_{\tau}(0)$ are positive and continuous functions of $\tau$,
this supremum is attained over the compact set $[0, A]$, leading to $c
> 0$. Since the upper bound holds uniformly over $\tau$, this shows
that for $t \in [0, t_1]$, we have $\sup_{\tau \in [0, A]} a_{t,
  \tau}^* = O(t)$.

Now, since for any $\tau \in [0, A]$, $a_{t, \tau}^*$ is a function of
class $C^1$ and thus continuous, its supremum over a compact set,
$\sup_{\tau \in [0, A]} a_{t, \tau}^*$, is also continuous and is
bounded over $[0, t_1]$ by some $a > 0$. For $|u| \leq a$ and $\tau
\in [0, A]$, we have $A - a \leq \tau + u \leq A + a$ and $A - a \leq
\tau - u \leq A + a$. Since $\Phi''$ is positive and continuous, it
reaches its maximum $C > 0$ over the compact set $[A - a, A + a]$.
Thus, we can write
\begin{align*}
  \forall t \in [0, t_1], \forall \tau \in [0, A], \quad
  \sT_\tau(t)
  & = \int_0^{a^*_{t, \tau}} u \bracket*{\frac{1 - t}{2} \Phi''_{\tau}(u) + \frac{1 + t}{2} \Phi''_{\tau}(-u)} \, du\\
  & \leq \int_0^{a^*_{t, \tau}} u \bracket*{\frac{1 - t}{2} C + \frac{1 + t}{2} C} \, du\\
  & = \int_0^{a^*_{t, \tau}} C u  \, du
  = C \frac{(a^{*}_{t, \tau})^2}{2}.
\end{align*}
Thus, for $t \leq t_1$, we have
\[
\inf_{\tau \in [0, A]} \sT_\tau(t) \leq C
\frac{(\sup_{\tau \in [0, A]} a^{*}_{t, \tau})^2}{2} \leq O(t^2).
\]
This completes the proof.
\end{proof}

\section{Analysis of the function of \texorpdfstring{$\tau$}{tau}}
\label{app:analysis}

Let $F$ be the function defined by
\[
\forall t \in \bracket*{0, \tfrac{1}{2}}, \tau \in \Rset, \quad
F(t, \tau)
= \sup_{u \in \Rset} \curl*{\Phi(\tau) - \frac{1 - t}{2} \Phi(\tau + u) - \frac{1 + t}{2} \Phi(\tau - u)},
\]
where $\Phi$ is a convex function in $C^2$ with $\Phi', \Phi'' >
0$. In light of the analysis of the previous sections, for any $(\tau,
t)$, there exists a unique function $a_{t, \tau}$ solution of the
maximization (supremum in $F$), a $C^1$ function over a neighborhood $U$
of $(\tau, 0)$ with $a_{0, \tau} = 0$, $a_{t, \tau} > 0$ for $t > 0$, and
$\frac{\partial a_{t, \tau}}{\partial t}(0, \tau)
= \frac{\Phi'(\tau)}{\Phi''(\tau)} = c_\tau$. Thus, we have $\lim_{t \to 0} \frac{a_{t, \tau}}{t c_\tau} = 1$. The optimality of
$a_{t, \tau}$ implies
\[
\frac{1 - t}{2} \Phi'(\tau + a_{t, \tau}) = \frac{1 + t}{2} \Phi'(\tau - a_{t, \tau}).
\]
Thus, the partial derivative of $F$ over the appropriate neighborhood $U$
is given by
\begin{align*}
  \frac{\partial F}{\partial \tau}(t, \tau)
  & = \Phi'(\tau) - \frac{1 - t}{2} \Phi'(\tau + a_{t, \tau}) \paren*{\frac{\partial a_{t, \tau}}{\partial \tau}(t, \tau) + 1} - \frac{1 + t}{2} \Phi'(\tau - a_{t, \tau}) \paren*{-\frac{\partial a_{t, \tau}}{\partial \tau}(t, \tau) + 1}\\
  & = \Phi'(\tau) - \frac{1 - t}{2} \Phi'(\tau + a_{t, \tau}) \paren*{\frac{\partial a_{t, \tau}}{\partial \tau}(t, \tau) + 1 - \frac{\partial a_{t, \tau}}{\partial \tau}(t, \tau) + 1} \\
  & = \Phi'(\tau) - (1 - t) \Phi'(\tau + a_{t, \tau}).
\end{align*}
% The second partial derivative of $F$ over the appropriate neighborhood $U$
% is given by
% \begin{align*}
%   \frac{\partial^2 F}{\partial \tau}(t, \tau)
%   & = \Phi''(\tau) - (1 - t) \Phi''(\tau + a_{t, \tau})\paren*{\frac{\partial a_{t, \tau}}{\partial \tau}(t, \tau) + 1}.
% \end{align*}
Since $\Phi'$ is continuous, by the mean value theorem, there exists
$\xi \in (\tau, \tau + a_{t, \tau})$ such that $\Phi'(\tau + a_{t,
  \tau}) - \Phi'(\tau) = a_{t, \tau} \Phi''(\xi)$. Thus, we can write
\begin{align*}
  \frac{\partial F}{\partial \tau}(t, \tau)
  & = \Phi'(\tau) - (1 - t) \Phi'(\tau) - (1 - t)  a_{t, \tau} \Phi''(\xi)\\
  & = t \Phi'(\tau) - (1 - t)  a_{t, \tau} \Phi''(\xi)\\
  & = t \Phi'(\tau) \bracket*{1 - (1 - t) \frac{a_{t, \tau}}{t c_\tau} \frac{\Phi''(\xi)}{ \Phi''(\tau)}}.
\end{align*}
Note that if $\Phi''$ is locally non-increasing, then we have $\Phi''(\xi) \leq \Phi''(\tau)$ and for $t$ sufficiently small, since $\Phi'$ is increasing and $\frac{a_{t, \tau}}{t c_\tau} \sim 1$:
\begin{align}
  \frac{\partial F}{\partial \tau}(t, \tau)
  & \geq t \Phi'(\tau) \bracket*{1 - (1 - t) \frac{a_{t, \tau}}{t c_\tau} }
  \geq 0.
\end{align}
In that case, for any $A > 0$, we can find a neighborhood $\sO$ of $t$
around zero over which $\frac{\partial F}{\partial \tau}(t, \tau)$ is
defined for all $(t, \tau) \in \sO \times [0, A]$ and $\frac{\partial
  F}{\partial \tau}(t, \tau) \geq 0$. From this, we can conclude that
the infimum of $F$ over $\tau \in [0, A]$ is reached at zero for $t$ sufficiently
small ($t \in \sO$).

\section{Generalization bounds}
\label{app:generalization-bound}

Let $S = \paren*{(x_1, y_1), \ldots, (x_m, y_m)}$ be a sample drawn
from $\sD^m$. Denote by $\h h_S$ an empirical minimizer within $\sH$
for the surrogate loss $\ell$: $ \h h_S \in \argmin_{h\in \sH}
\frac{1}{m}\sum_{i = 1}^m \ell (h, x_i, y_i)$. Let
$\sH_{\ell}$ denote the hypothesis set $\curl*{(x, y) \mapsto \ell(h, x, y)
  \colon h \in \sH}$ and $\Rad_m^{\ell}(\sH)$ its Rademacher
complexity. We also write $B_{\ell}$ to denote an upper bound for
$\ell$. Then, given the following $\sH$-consistency bound:
\begin{equation}
\label{eq:H-consistency-bounds}
\forall h \in \sH, \quad
\sE_{\ell_{0-1}}(h) - \sE^*_{\ell_{0-1}}(\sH) + \sM_{\ell_{0-1}}(\sH)
\leq \Gamma\paren*{\sE_{\ell}(h)-\sE^*_{\ell}(\sH) + \sM_{\ell}(\sH)},
\end{equation}
for any
$\delta > 0$, with probability at least $1 - \delta$ over the draw of an
i.i.d.\ sample $S$ of size $m$, the following estimation
bound holds for $\h h_S$:
\begin{equation*}
\forall h \in \sH, \quad
\sE_{\ell_{0-1}}(h) - \sE^*_{\ell_{0-1}}(\sH) 
\leq \Gamma\paren[\bigg]{4
    \Rad_m^{\sfL}(\sH) + 2 B_{\sfL} \sqrt{\tfrac{\log
        \frac{2}{\delta}}{2m}} + \sM_{\ell}(\sH)} - \sM_{\ell_{0-1}}(\sH).
\end{equation*}
\begin{proof}
  By the standard Rademacher complexity bounds
  \citep{MohriRostamizadehTalwalkar2018}, for any $\delta > 0$, with
  probability at least $1 - \delta$, the following holds for all $h
  \in \sH$:
\[
\abs*{\sE_{\ell}(h) - \h\sE_{\ell,S}(h)}
\leq 2 \Rad_m^{\ell}(\sH) +
B_{\ell} \sqrt{\tfrac{\log (2/\delta)}{2m}}.
\]
For any $\e > 0$, by definition of the infimum, there exists $h^* \in
\sH$ such that $\sE_{\ell}(h^*) \leq
\sE_{\ell}^*(\sH) + \e$. By the definition of
$\h h_S$, we obtain
\begin{align*}
  \sE_{\ell}(\h h_S) - \sE_{\ell}^*(\sH)
  & = \sE_{\ell}(\h h_S) - \h\sE_{\ell,S}(\h h_S) + \h\sE_{\ell,S}(\h h_S) - \sE_{\ell}^*(\sH)\\
  & \leq \sE_{\ell}(\h h_S) - \h\sE_{\ell,S}(\h h_S) + \h\sE_{\ell,S}(h^*) - \sE_{\ell}^*(\sH)\\
  & \leq \sE_{\ell}(\h h_S) - \h\sE_{\ell,S}(\h h_S) + \h\sE_{\ell,S}(h^*) - \sE_{\ell}^*(h^*) + \e\\
  & \leq
  2 \bracket*{2 \Rad_m^{\ell}(\sH) +
B_{\ell} \sqrt{\tfrac{\log (2/\delta)}{2m}}} + \e.
\end{align*}
Since the inequality holds for all $\e > 0$, it implies the following:
\[
\sE_{\ell}(\h h_S) - \sE_{\ell}^*(\sH)
\leq 
4 \Rad_m^{\ell}(\sH) +
2 B_{\ell} \sqrt{\tfrac{\log (2/\delta)}{2m}}.
\]
Plugging in this inequality in the $\sH$-consistency bound
\eqref{eq:H-consistency-bounds} completes the proof.
\end{proof}
These bounds for surrogate loss minimizers, expressed in terms of
minimizability gaps, offer more detailed and informative insights
compared to existing bounds based solely on approximation errors. Our
analysis of growth rates suggests that for commonly used smooth loss
functions, $\Gamma$ varies near zero with a square-root dependency.
Furthermore, this dependency cannot be generally improved for
arbitrary distributions.

\section{Future work}
\label{app:future_work}

While we established a universal square-root growth rate for the widely used class of smooth surrogate losses in both binary and multi-class classification, our results hold for any distribution. Further research on cases with distributional assumptions would be an interesting direction.

\end{document}